\documentclass{article}

\usepackage{microtype}
\usepackage{graphicx}
\usepackage{subfigure}
\usepackage{booktabs} %

\usepackage{enumerate, amsmath, amsthm, amssymb, pifont, csquotes, dashrule, mathrsfs, bbm, booktabs, bm, verbatim}
\usepackage[framemethod=TikZ]{mdframed}

\definecolor{dblue}{RGB}{98, 140, 190}
\definecolor{dlblue}{RGB}{216, 235, 255}
\definecolor{dgreen}{RGB}{124, 155, 127}
\definecolor{dpink}{RGB}{207, 166, 208}
\definecolor{dyellow}{RGB}{255, 248, 199}
\definecolor{dgray}{RGB}{46, 49, 49}

\newcommand{\durl}[1]{\textcolor{dblue}{\underline{\url{#1}}}}

\newmdenv[
  topline=false,
  bottomline=false,
  rightline = false,
  leftmargin=10pt,
  rightmargin=0pt,
  innertopmargin=0pt,
  innerbottommargin=0pt
]{innerproof}

\newcounter{KDefCounter}
\newcommand{\ddef}[2]
{
\vspace{1mm}
{\bf Definition \theKDefCounter} (#1): {\it #2}
\stepcounter{KDefCounter}
}

\newtheorem{lemma}{Lemma}

\newtheorem{remark}{Remark}
\newtheorem{theorem}{Theorem}

\usepackage{hyperref}

\usepackage{xr} %
\makeatletter

\usepackage[accepted]{icml2020}
\usepackage{selectp}
\icmltitlerunning{What can I do here? A Theory of Affordances in Reinforcement Learning}

\newcommand{\AF}{{{\mathcal{A} \mathcal{F}}_{\cal I}}}
\newcommand{\AFP}{{{\mathcal{A} \mathcal{F}}_{\cal I'}}}

\begin{document}

\twocolumn[
\icmltitle{What can I do here? A Theory of Affordances in Reinforcement Learning}

\icmlsetsymbol{note}{*}
\begin{icmlauthorlist}
\icmlauthor{Khimya Khetarpal}{mila,note}
\icmlauthor{Zafarali Ahmed}{deepmind}
\icmlauthor{Gheorghe Comanici}{deepmind}
\icmlauthor{David Abel}{brown}
\icmlauthor{Doina Precup}{mila,deepmind}
\end{icmlauthorlist}
\icmlaffiliation{mila}{Mila - McGill University}
\icmlaffiliation{deepmind}{DeepMind}
\icmlaffiliation{brown}{Brown University}

\icmlcorrespondingauthor{Khimya Khetarpal}{khimya.khetarpal@mail.mcgill.ca}

\icmlkeywords{Machine Learning, ICML}

\vskip 0.3in
]

\printAffiliationsAndNotice{\icmlAdditionalNote} %

\begin{abstract}
Reinforcement learning algorithms usually assume that all actions are always available to an agent. However, both people and animals understand the general link between the features of their environment and the actions that are feasible. Gibson (1977) coined the term ``affordances'' to describe the fact that certain states enable an agent to do certain actions, in the context of embodied agents. In this paper, we develop a theory of affordances for agents who learn and plan in Markov Decision Processes. Affordances play a dual role in this case. On one hand, they allow faster planning, by reducing the number of actions available in any given situation. On the other hand, they facilitate more efficient and precise learning of transition models from data, especially when such models require function approximation. We establish these properties through theoretical results as well as illustrative examples. We also propose an approach to learn affordances and use it to estimate transition models that are simpler and generalize better.
\end{abstract}

\section{Introduction}
\label{submission}

Humans and animals have an exceptional ability to perceive their surroundings and understand which behaviors can be carried out successfully. For example, a hard surface enables walking or running, whereas a slippery surface enables skating or sliding. 
This capacity to focus on the most relevant behaviors in a given situation enables efficient decision making by limiting the choices of action that are even considered, and leads to quick adaptation to changes in the environment.

\citet{gibson1977theory} defined \textit{affordances} as different possibilities of action that the environment \textit{affords} to an agent. For example, water affords the action of swimming to a fish, but not to a land animal. Hence, affordances are a function of the environment as well as the agent, and \textit{emerge} out of their interaction. \citet{heft1989affordances} discussed the fact that affordances are located at the agent-environment boundary.
\citet{gibson1977theory} pointed out that affordances can also be viewed as a way to characterize an agent's state in an action-oriented fashion. For example, a seat can be any object on which one can sit above from the ground, regardless of its shape or color. This view leads potentially to very robust generalization when processing the perceptual stream in order to determine what to do. We take inspiration from \citet{gibson1977theory}, \citet{heft1989affordances}, and \citet{chemero2003outline} and provide a framework that enables artificially intelligent (AI) agents to represent and reason about their environment through the lens of affordances.

In this paper, we focus on reinforcement learning (RL) agents~\cite{sutton2018introduction}. We aim to endow RL agents with the ability to represent and learn affordances, which can help them to plan more efficiently, and lead to better generalization. While defining affordances is not a new topic in AI (see Sec.~\ref{sec:related} for a discussion of related work), our approach builds directly on the general framework of Markov Decision Processes (MDPs), in its traditional form. 

In order to define affordances, we need to first capture the notion of what it would mean for an agent to carry out an action ``successfully". To do this, we introduce the notion of {\em intent}, i.e., a desired outcome for an action. 
Affordances will then capture a subset of the state-action space in which the intent is achieved. This view is very compatible with model-based RL, in which the transition model captures the consequences of actions. However, learning an accurate model of the entire environment can be quite difficult, especially in large environments. Hence, we propose to learn affordances, and use them to define {\em partial models} \cite{talvitie2009simple}, which  focus on making high quality predictions for a subset of state and actions: those linked through an affordance..

We first define affordances in MDPs (Sec.~\ref{sec:afford}) and quantify the value loss when replacing the true MDP model with an affordance-based model (Sec.~\ref{sec:valueloss}). Then, we investigate the setting in which affordance-based partial models are learned from data, and we show that the planning loss is bounded, with high probability (Sec.~\ref{sec:planningloss}). The bound is given in terms of the complexity of the policy class determined by the size of the affordances. We provide empirical illustrations for this analysis (Sec.~\ref{sec:experiments}).
Finally, we propose an approach to learn affordances from data and use it to estimate a partial model of the world (Sec.~\ref{sec:experiments_learning}). Our results provide evidence that affordances and partial models lead to improved generalization and stability in the learning process.

\section{Background}
\label{sec:background}
In reinforcement learning (RL), a decision-making agent must learn to interact with an environment, through a sequence of actions, in order to maximize its expected long-term return~\cite{sutton2018introduction}. This interaction is typically formalized using the framework of Markov Decision Processes (MDPs). An MDP is a tuple $M=\langle {\cal S}, {\cal A}, r, P, \gamma \rangle $, where ${\cal S}$ is a set of states, ${\cal A}$ is a set of actions, $r: {\cal S} \times {\cal A}\rightarrow [0, R_{\max}$] is the reward function, $P:{\cal S} \times {\cal A} \rightarrow Dist({\cal S})$ is the environment's transition dynamics, mapping state-action pairs to a distribution over next states,  $Dist(\mathcal{S})$, and  $\gamma \in (0,1)$ is the discount factor.  At each time step $t$, the agent observes a state $s_t \in {\cal S}$ and takes an action $a_t \in {\cal A}$ drawn from a policy $\pi : {\cal S}\rightarrow  Dist({\cal A}) $. Then, with probability $P(s_{t+1}|s_t,a_t)$, the agent enters the next state $s_{t+1}\in{\cal S}$, receiving a numerical reward $r(s_t, a_t)$. The value function for a policy $\pi$ is defined as: 
$V_\pi(s) = E\left[\sum_{t=0}^{\infty} \gamma^{t} r(s_t, a_t) \big| s_0 = s, a_t \sim \pi(\cdot|s_t),\forall t\right].$
For simplicity of exposition, we assume henceforth that the MDP's state and action space are finite, though the ideas we present can be extended naturally to the infinite setting.

The goal of an agent is to find the optimal policy, $\pi^*=\arg\max_{\pi} V^{\pi}$ (which exists in a finite MDP). If the model of the MDP, consisting of $r$ and $P$, is given, the value iteration algorithm can be used to obtain the optimal value function, $V^*$, by computing the fixed-point of the system of Bellman equations~\cite{bellmann1957dynamic}: $V^*(s) = \max_a \Big( r(s,a) + \gamma \sum_{s^{'}} P(s'|s,a) V^*(s^{'}) \Big), \forall s.$ The optimal policy $\pi^*$ can be obtained by acting greedily with respect to $V^*$.

Because in general the true model of the environment is unknown, one approach that can be used to solve the optimization above is to use data in order to construct an approximate model, $\langle\hat{r}, \hat{P}\rangle$, usually by using maximum likelihood estimation, then solve the corresponding approximate MDP $\hat{M}$. This approach is called {\em model-based RL} or {\em certainty-equivalence (CE) control}. Given a finite amount of data, the estimate of the model will be inaccurate, and thus we will be interested in evaluating the optimal policy obtained in this way, $\pi^*_{\hat{M}}$, in the true MDP $M$. We denote by $V^{\pi}_M$ the value of any policy $\pi$ when evaluated in $M$. Our results will bound the value loss of various policies computed from an  approximate MDP compared to the true optimal value, in some $\ell_p$-norm,  $||V^{\pi_{\hat{M}}}_M- V^*_M||_p$. 

\section{Affordances}\label{sec:afford}
In an MDP, we usually assume that all actions are available in all states, and the model $\langle r,P\rangle$ therefore has to be defined for all $(s,a)$. We now build a framework for defining and using affordances, which limit the state-action space of interest. For this, we need to formalize Gibson's intuition of ``action success". Because we would like affordances to generalize across environments with different rewards, we start by considering a notion of ``intent" of an action, and we consider an action to have succeeded if it realizes its intent. For example, having a coffee machine affords the action of making coffee, because we can successfully obtain coffee from the machine. This does not necessarily mean the action is desirable in the current context: if the agent must go to sleep soon, or has an upset stomach, the reward for drinking coffee might be negative. Nonetheless, the action itself can be executed and would result in the intended consequence of possessing coffee. This example gives the intuition that intent is best captured by thinking about a target state distribution that should be achieved after executing an action.

\ddef{Intent $I_a$}{Given an MDP $M$ and action $a \in {\cal A}$, an intent is a map from states to desired state distributions that should be obtained after executing $a$, $I_a:{\cal S} \rightarrow \text{Dist}(\mathcal{S})$. An intent $I_a$ is satisfied to a degree, $\epsilon$, at state $s\in \cal S$ if and only if:
\begin{equation}
    \label{def:intent_state}
   d(I_a(s), P(\cdot|s,a))\leq \epsilon,
\end{equation}
where $d$ is a metric between probability distributions and $\epsilon\in [0,1]$ is a desired precision.
An intent $I_a$ is satisfied to a degree $\epsilon$  in MDP $M$ if and only if it is satisfied to degree $\epsilon$ $\forall s\in \cal S$.
}

In this work, we will take $d$ to be the total variation metric: $d(P,Q)=\frac{1}{2}\sum_{x\in \cal{X}} |P(x)-Q(x)|$
where ${\cal X}$ is assumed to be discrete and finite. Note that $d=0$ iff $P$ and $Q$ are identical, and the maximum value possible for $d$ is $1$ (which will make it convenient to use $\epsilon$).

This notion of intent is similar to empowerment~\cite{salge2014empowerment} or to setting a goal for the agent to reach after a temporally extended action \cite{schaul2015uvfa,nachum2018data}. Note that if the intent maps all states to the {\em same} distribution, we capture a strong notion of invariance (a kind of ``funneling"), so that the result of the action is insensitive to the state in which it is executed.
 
Based on this definition, it is clear that any intent can be satisfied if we set $\epsilon=1$. Depending on $\epsilon$, there may be no way to satisfy an intent for an action, given that the conditioning is over {\em all} states in the MDP. However, for the intents that were satisfied, we could imagine the agent planning in an approximate MDP in which $I_a$ replaces the transition model, because action $a$ will reliably take states into a known distribution, when $a$ completes. If $I_a$ happened to have support only at one state, i.e. $a$ is deterministic, then the planning could also be really efficient. 
 
For example, in a navigation task like the one depicted in Fig.~\ref{fig:Illustration}, an agent that intends to move left can do so in many states, but not in the ones that are immediately adjacent to the left wall. We build the notion of affordances with the goal of restricting the state-action pairs so as to allow intents to be satisfied.

\ddef{Affordances $\AF$}{Given a finite MDP, a set of intents ${\cal I}=\cup_{a\in{\cal A}} I_a$, and $\epsilon\in[0,1]$, we define the affordance $\AF$ associated with ${\cal I}$ as a relation  $\AF \subseteq \cal S \times \cal A$, such that $\forall (s,a) \in \AF$, $I_a$ is satisfied to degree $\epsilon$ in $s$.}

We will use $\AF(s) \subset {\cal A}$ to denote the set of actions $a$ such that $(s,a)\in \AF$.%

 \begin{figure}[t]
    \begin{center}
    \subfigure[Move Intent]{\label{fig:intent_changestate}\includegraphics[width=0.2\textwidth]{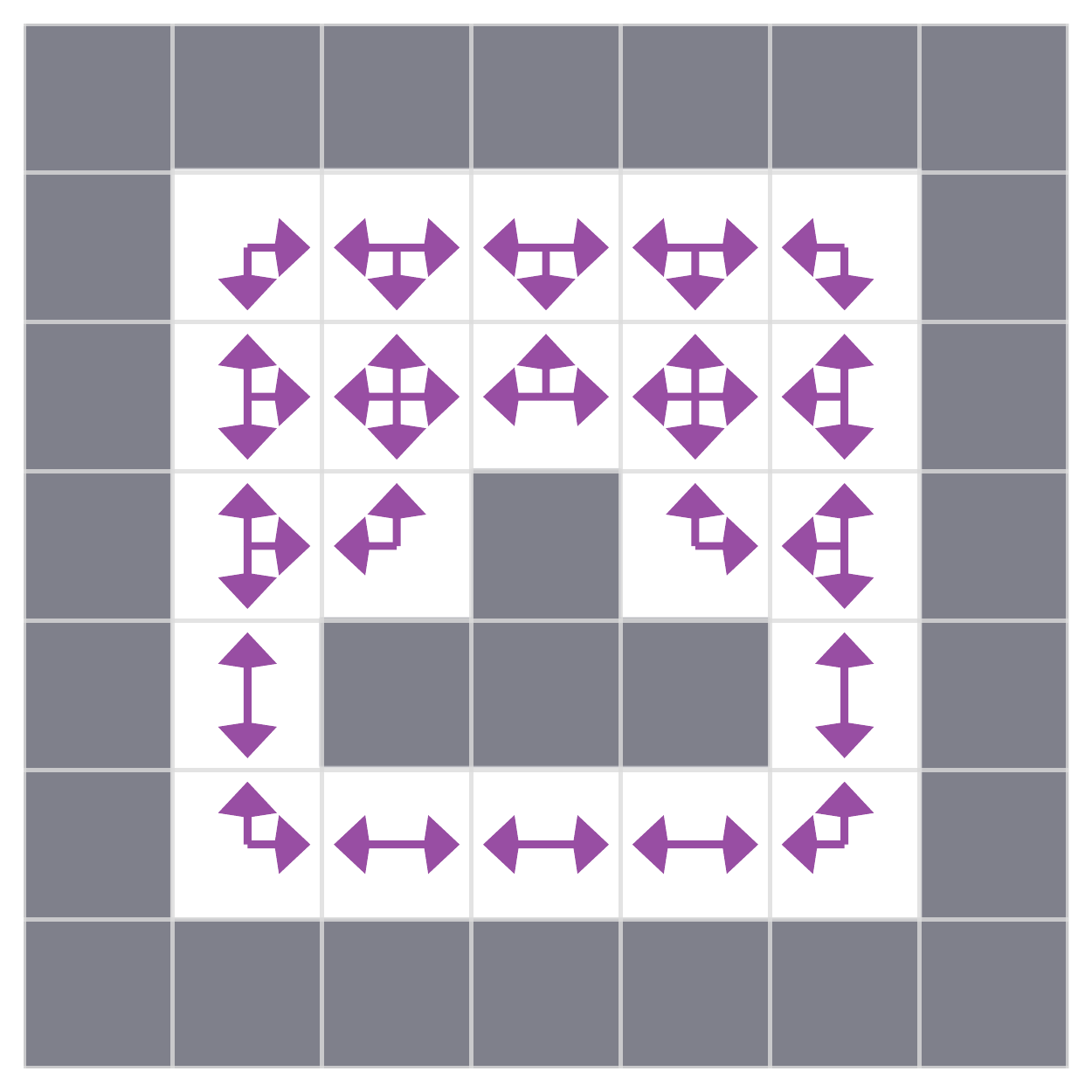}}\vspace*{-0.4cm}
    \hspace{0.4cm}
    \subfigure[Move Left Intent]{\label{fig:intent_left}\includegraphics[width=0.2\textwidth]{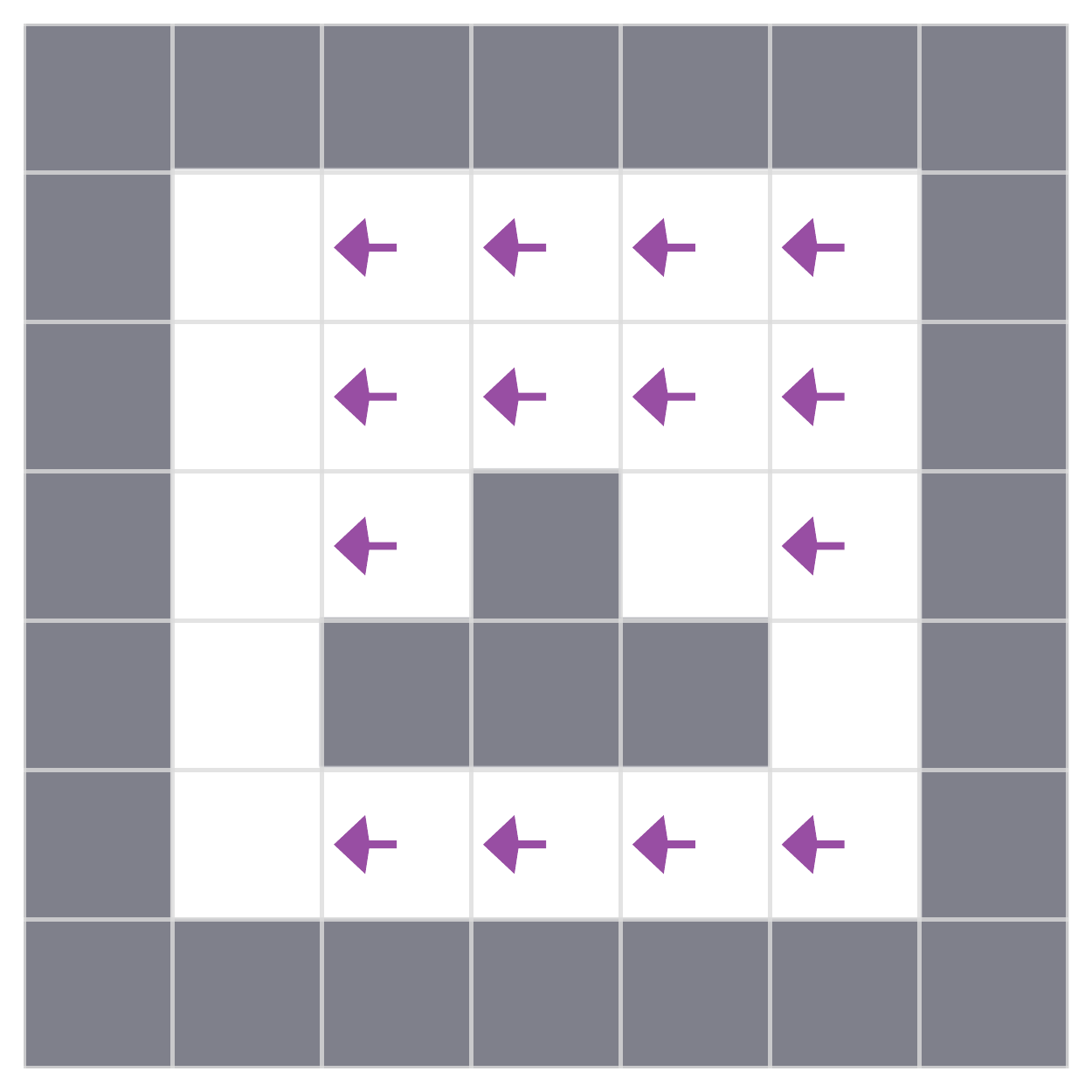}\vspace*{-0.5cm}}\hspace{-0.4cm}
    \caption{\label{fig:Illustration}\textbf{Visualization of affordances in a grid-world.} \textit{Affordances} are the subset of states and actions which satisfy the \textit{intents} to a desired degree (Cf. Def~\ref{def:intent_state}). Affordances are shown for intents specified as bringing (a) any change in position $\Delta y$ or $\Delta x$, and (b) a change of $-\Delta x$ only. Grid cells and arrows in each cell represent the states and actions respectively. See appendix for more illustrations.}
    \end{center}
    \vskip -0.2in %
\end{figure}

\textbf{Illustration}
To illustrate the idea of intents and affordances, consider the navigation task depicted in Fig.~\ref{fig:Illustration}, in which the agent always has 4 available actions that move it deterministically to a neighbouring state, unless the agent is next to a wall, in which case it remains in the same position. Consider the intent to be a change in the agent's current position. Given this intent specification, the affordance corresponding to any $\epsilon<1$ contains the subset of state-action pairs marked in Fig~\ref{fig:intent_changestate}, i.e. the grid cells and actions which successfully change the agent's position. It excludes state-action pairs which move the agent into a wall. Note that if the probability of ``success" of an action were $p_{succ}$, any $\epsilon<p_{succ}$ would result in the same affordance. Similarly, the affordance corresponding to the intent of moving left is depicted in Fig~\ref{fig:intent_left}. Note that it is possible to obtain an empty affordance for values of $\epsilon$ that are too stringent. We will further examine the effect of $\epsilon$ in our experiments.

Our definitions assume that an agent starts from an intent specification, which could either be given {\it a priori}, specified by a human in the loop, given by a planner, or learned and adapted over time. Affordances are then constructed on the basis of the intents. Our definitions are intended to allow RL agents to capture dynamics that are consistent and invariant to various factors in the environment. For example, in navigation tasks, we would like the agent to handle intents and affordances that allow its models to be robust with respect to variations such as the location of the walls, or the exact shape and size of the rooms. Examples of such invariances are given in Fig.~\ref{fig:appendix_Illustration}. %

\section{Value Loss Analysis}\label{sec:valueloss}

Given an MDP $M$ and set of intents $\mathcal{I}$, it is easy to notice that we can define an {\em induced MDP} $M_{\cal I} =\langle {\cal S}, {\cal A}, r, P_{\mathcal{I}}, \gamma \rangle$, where $r$ is the same as in the original MDP, and the set of intents induces the transition model $P_{\mathcal{I}}$. Since the intent specifies a desired distribution over next states, we can assume that from the states that afford a specific action $a$, the intent is a close-enough approximation of the transition model to be used ins stead. In all other states, action $a$ does not need to be considered, so we do not need to model it. 
We now study the value loss incurred due to using a model $M_{\cal I}$ based on intents as a proxy for the true model $M$.
\begin{theorem}
\label{thm:value_loss_analysis}
Let $\mathcal{I}$ be a set of intents and $\epsilon_{s,a}$ be the minimum degree to which an intent is satisfied for $(s,a)$.
\begin{equation}
\sum_{s'} \Big| P_{\mathcal{I}}(s'|s,a) - P(s'|s,a) \Big| \leq \epsilon_{{s,a}}.
\label{eq:epsilon}
\end{equation}
Let $\epsilon=\max_{s,a} \epsilon_{s,a}$. Then, the value loss between the optimal policy for the original MDP $M$ and the optimal policy $\pi^*_\mathcal{I}$ computed from  the induced MDP $M_{\cal I}$ is given by:
\begin{equation}
    ||V^{\pi^{*}_\mathcal{I}}_{M} - V^*_M||_{\infty} \leq   2\epsilon \frac{\gamma \texttt{Rmax}}{(1-\gamma)^{2}},  %
\label{eq:theorem1}
\end{equation}
where $\texttt{Rmax}$ is the maximum possible value of the reward.
\end{theorem}
The proof is provided in appendix~\ref{sec:appendix_valuelossbound}.

\section{Planning Loss Bound}
\label{sec:planningloss}

So far we considered the effect of using the intent-based MDP in order to plan.
However, we would also like to use the associated affordance set, $\AF$, in order to speed up the planning process. Moreover, we would like to consider fine-tuning the intent set $\cal I$ by using data. 

We consider the certainty-equivalence control setting (as described in Sec~\ref{sec:background}), in which we optimize a policy based on an approximate model $\hat{M}_{\AF}$. More precisely, we use data to approximate the model of the MDP, but only for the set of state-action pairs in the affordance $\AF$ induced by a given set of intents $\cal I$ and a given $\epsilon$. State-action pairs which are not in $\AF$ will be considered impossible.

For simplicity, we assume that $\epsilon$ is such that  $|\AF(s)|\geq 1, \forall s\in \cal S$ (in other words, there is at least one action available in each state).

Note that $\hat{M}_{\AF}$ will have the same reward function $r$ as the original MDP $M$, but instead of $\cal S \times \cal A$, its model will be defined on $\AF$. One can think of this MDP as working with {\em partial models}~\cite{talvitie2009simple}, which are defined only for some state-action pairs.

Let $\pi^*_{\hat{M}_\AF}$ be the optimal policy of MDP $\hat{M}_\AF$. We quantify the largest absolute difference (over states) between the value of the true optimal policy with respect to the true model, $\pi^*_{M}$ and that of $\pi^*_{\hat{M}_\AF}$ when evaluated in $M$:
\begin{equation}
    \text{\textbf{Planning Value Loss: }} \Big| \Big|  V^*_{M} - V^{\pi^*_{\hat{M}_\AF}}_M \Big| \Big|_{\infty}
\end{equation}
Our work builds on the theory developed by \citet{jiang2015dependence}, which characterizes a bias-variance trade-off in approximate planning based on the complexity of the policy class allowed. 
\citet{jiang2015dependence} suggest that $\gamma$ can be viewed as a parameter that controls the number of policies that can be optimal, given a fixed state-action space along with a reward function. The authors draw parallel to supervised learning and their theory is suggestive of the fact that limiting the complexity of the policy class by using $\gamma$ can lead to an optimal bias-variance trade-off for a fixed amount of data.

Note that intuitively, the policy class in our case will depend on the affordances. For example, if only a single action can be carried out at any given state, there is only one policy available. If all actions are always available, the policy class is the same as in the original MDP $M$. Hence, the ``size" of the affordance controls the policy class, and of course, this size depends on $\epsilon$ and on the intent set $\cal I$.
 
We will now define the policy class for affordances as follows:

\ddef{Policy class $\Pi_{\mathcal{I}}$}{Given affordance $ \AF$, let  $\mathcal{M_I}$ be the set of MDPs over the state-action pairs in $\AF$, and let 
\begin{equation*}
\Pi_{\mathcal{I}} = \{ \pi^{*}_M \} \cup \{ \pi : \exists \bar M \in  \mathcal{M_I} \text{ s.t. }  \pi \text{ is optimal in } \bar M  \}.
\end{equation*}}
Affordances and intents control the size of the policy class as highlighted in the following remark. 
\begin{remark}
Given a set of intents $\cal I$, the following statements  hold for affordance $\AF$ and their corresponding policy class $\Pi_{\mathcal{I}}$:
\begin{enumerate}
    \item
     If  $ \forall{ s \in S}, |\AF(s)| = 1$ (i.e. only one action is affordable at every state), then $|\Pi_{\mathcal{I}}| = 1$.
    \item $|\AF| \leq |\AFP| \implies |\Pi_{\mathcal{I}}| \leq |\Pi_{\mathcal{I'}}|$.
\end{enumerate}
\label{remark1}
\end{remark}
We now present our main result. We show that the loss of the  policy for ${\hat{M}_\AF}$ is bounded, with high probability in terms of the policy class complexity $|\Pi_{\cal I}|$, which is controlled by the size of the affordance $\AF$. 
\begin{theorem}
Let $\hat{M}_\AF$ be the approximate MDP over affordable state-action pairs. Then certainty equivalent planning with $\hat{M}_\AF$ has planning loss 
\begin{align*}
    \Big|\Big|V^*_M - V^{\pi^{*}_{\hat{M}_\AF}}_M \Big|\Big|_{\infty} \leq& \frac{2 \texttt{Rmax}}{(1-\gamma)^2} \hookleftarrow\\
    &\times\Bigg(2\gamma \epsilon + \sqrt{\frac{1}{2n} \log \frac{2 |\AF| |\Pi_\mathcal{I}|}{\delta}}\Bigg)
\end{align*}
with probability at least $1-\delta$.
\label{theorem:planningvaluelossbound}
\end{theorem}

The proof is in appendix~\ref{sec:appendix_planninglossbound}. Our result has similar implications to \citet{jiang2015dependence}: for a given amount of data, there will be an optimal affordance size $|\AF|$ that will provide the best bias-variance trade-off. Note also that as the amount of data used to estimate the model increases, the bound shrinks, so the affordance and class of policies can grow. Intuitively, the planning value loss now becomes a trade off between the size of affordances $|\AF|$, the corresponding policy space  $|\Pi_\mathcal{I}|$ and the approximation error, $\epsilon$. We could be very precise in the choice of intents, but then the affordance set becomes larger, hence the bound is looser. Alternatively, we could make the affordance set much smaller, but that might come at the expense of a poor approximation in the intent model, thereby loosening the bound.

\section{Empirical Results}\label{sec:experiments}
In this section, we conduct a set of experiments to illustrate our theoretical results\footnote{We provide the source code for all our experimental results at \url{https://tinyurl.com/y9xkheme}.}. In Sec.~\ref{sec:implact_affordance_planning}, we study the effect of planning with partial models constructed using intents and affordances on the quality of the plans.
In Sec~\ref{sec:planning_time}, we illustrate the potential of affordances to accelerate planning. Finally, in Sec~\ref{sec:acc_effectice_affordances} we study the planning accuracy when we use affordances and data to construct an approximate $\hat{M}_\AF$ (as opposed to using $\cal I$ as an approximate model in zero-shot fashion).

\subsection{Planning with Intents}
\label{sec:implact_affordance_planning}
\begin{figure}[t]
\centering
\includegraphics[trim={0 0.1cm 0 0},clip,width=0.8\columnwidth]{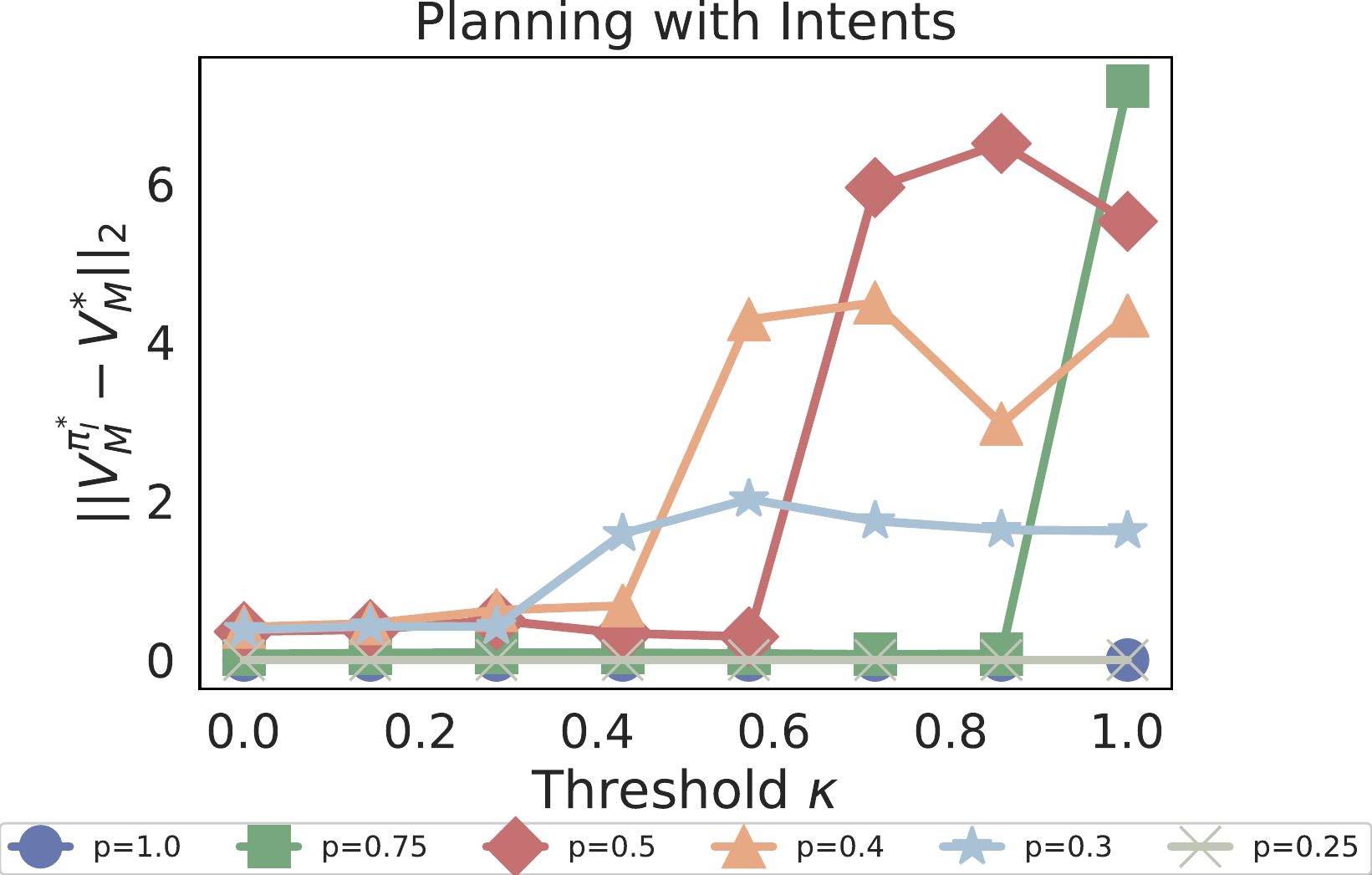}
\caption{\label{fig:planningwithintents}\textbf{The impact of intents and affordances on planning.} The environment is as in Fig.~\ref{fig:Illustration}. The actions are stochastic and fail with a probability of $1 - p$, which is varied in the experiments. $\kappa=0.0$ will provide maximum coverage of ${\mathcal S} \times {\mathcal A}$ and therefore results in performance close to the optimal policy's performance  for all values of p. As $\kappa$ increases, the affordance becomes more selective and the loss increases. The effect is not uniform because the intents are at the same time more accurate. %
}
\vskip -0.1in %
\end{figure}

We first study the impact of planning with  given affordances, as the degree of stochasticity in the environment changes. 

\textit{Experimental Setup:} We consider the gridworld depicted in Fig.~\ref{fig:Illustration}, where the actions are modified to be stochastic and fail with probability $1 - p$. The agent starts in the bottom left state, and the goal is situated in the top right state. Rewards are all 0, except at the goal where the reward is 1. We pick a collection of intents that describe successful movement in the directions of the different actions, and compute the affordances for different values of $\epsilon$. For the ease of plotting, in these and the following graphs we will use on the x-axis $\kappa=1-\epsilon$. Hence, for $\kappa=0$, all state-action pairs are in the affordance. As $\kappa$ increases, the affordance becomes smaller. We build the intent-induced MDP $M_{\cal I}$ as described in Sec.~\ref{sec:valueloss}, but the transition probabilities are limited only to state-action pairs in $\AF$.

Intuitively, as the affordance becomes smaller (by increasing $\kappa$), we will obtain an MDP in which $\cal I$ is more precise, but which limits the number of actions available at each state. Hence, it is interesting to inspect the trade-off between affordance size and value loss. We run value iteration in the two MDPs, $M$ and $M_{\cal I}$, to obtain the optimal policies $\pi^{*}_M$ and $\pi^{*}_{M_{\cal I}}$ respectively. We then evaluate and plot $||V^*_M - V^{\pi^*_{\cal I}}_M||_{2}$.

As the stochasticity in the actions decreases (higher values of $p$), a higher value of $\kappa$ results in more bias, due to reducing the action space too much. If the actions are deterministic ($p=1$), the affordance covers all state-action pairs regardless of the threshold, and the intents always match the real model, resulting in 0 loss everywhere (Fig.~\ref{fig:planningwithintents}, blue squares). Note that the curves are non-monotonic, because of the trade-off of using a better intent, which leads to more precise planning compared to reducing the state-action space which can introduce systematic errors.

\begin{figure}[t]
    \vskip -0.1in %
    \begin{center}
    \subfigure[One-room]{\label{fig:tworooms_plantime}\includegraphics[trim={0, 0.5cm, 0, 0.5cm},clip,width=0.23 \textwidth]{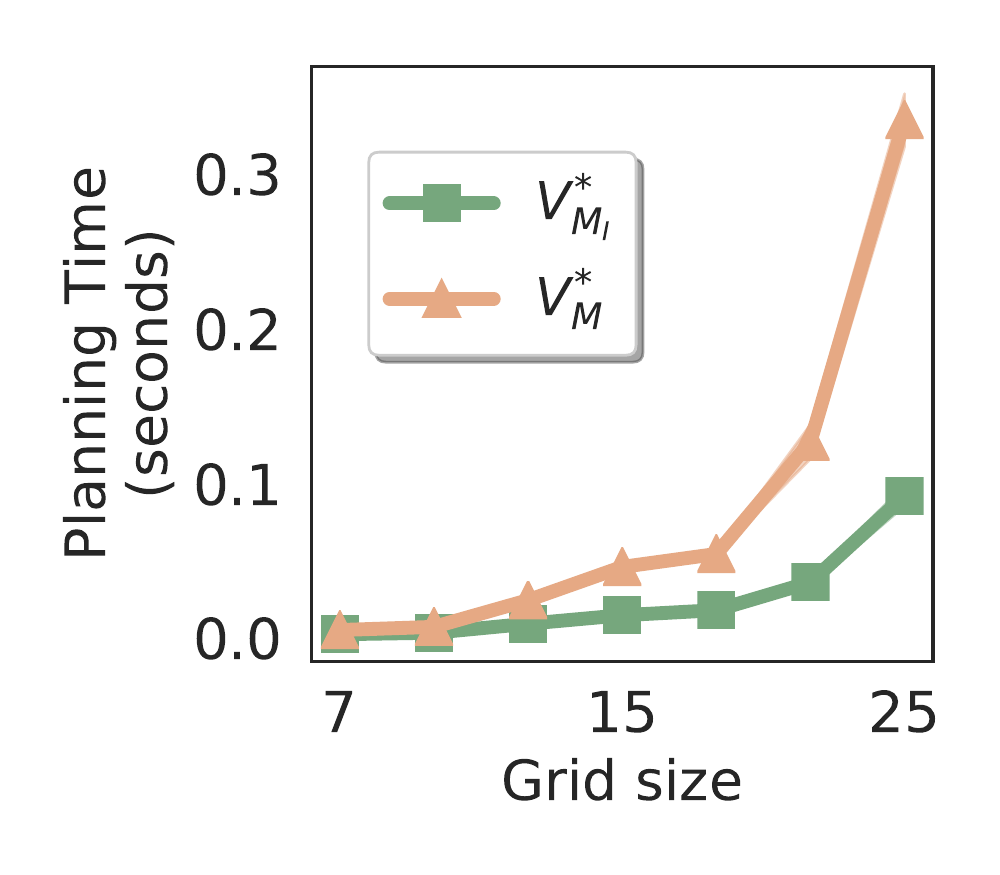}
    }
    \subfigure[Pachinko]{\label{fig:pachinko_plantime}\includegraphics[trim={0, 0.5cm, 0, 0.5cm},clip,width=0.23 \textwidth]{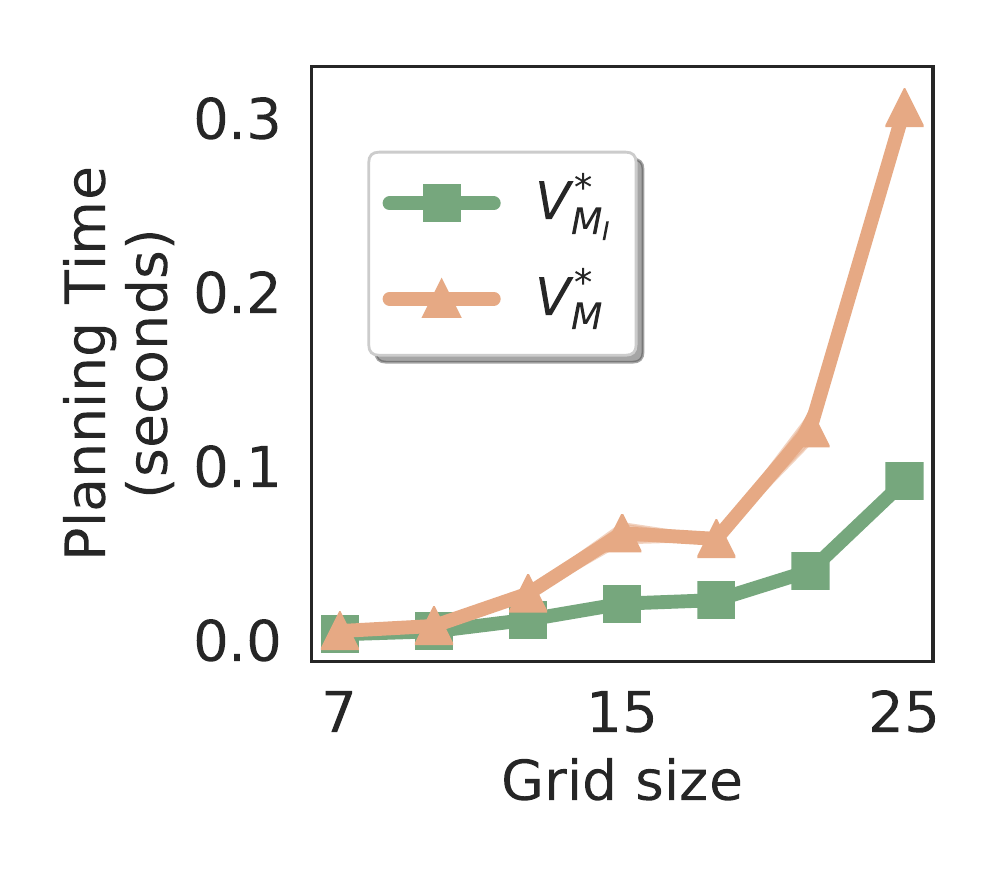}}
    \caption{\label{fig:planningtime}\textbf{Planning time for Value Iteration with and without affordances as a function of grid size.} Planning is significantly quicker with an affordance-aware model as the size of the grid increases. The shaded areas represent the standard error in the mean over 10 independent runs.}
    \end{center}
    \vskip -0.2in %
\end{figure}

\subsection{Accelerated Planning with Affordances}
\label{sec:planning_time}

To quantify the benefits of using a reduced state-action space through affordances, we investigate the running time of value iteration when planning in $M_{\cal I}$, compared to planning with the true model $M$. 

\textit{Experimental Setup:} We run value iteration in two MDPs, Pachinko and One-room (depicted in Fig.~\ref{fig:appendix_Illustration}), and simulate progressively difficult problems by increasing the size of the grid from $7$ to $25$. In both environments, the actions fail with probability $0.5$, lead to a neighbouring state chosen uniformly at random. In Pachinko, the wall configuration restricts paths through the maze. We use the same intents as before, and a threshold $\kappa=0.5$ to build the affordance. We measure the planning time as the time taken for the value iteration updates to be below a given, small threshold. 

Using the affordances significantly reduces the planning time, compared to using the full model, especially for the environments larger than size 15 (Fig.~\ref{fig:planningtime}). Hence, appropriate affordances can result in planning more efficiently. 

\subsection{Planning Loss with Affordances and Learned Models}
\label{sec:acc_effectice_affordances}

Thm.~\ref{theorem:planningvaluelossbound} shows that the planning value loss depends on the size of the affordances, $|\AF|$, and the amount of data used to estimate the model. We now study empirically the planning value loss for varying amounts of data and affordance size.

\textit{Experimental Setup:} For this experiment we use a $19 \times 19$ Pachinko grid world (Fig.~\ref{fig:appendix_Illustration}). The probability of success of the actions is drawn uniformly at random from $[0.1, 1]$ for each state. We estimate the model $\hat{M}_\AF$ from the data which is generated by randomly sampling a state $s$ and then taking a sequence of 10 uniformly random actions. For each state-action pair $(s,a)$ which has not been visited, the distribution over the next state, $P(\cdot|s, a)$, is initialized uniformly.%

We observe in Fig.~\ref{fig:planning_value_loss} that for the small data regime ($n = 25-200$ trajectories), the minimum planning loss is achieved at intermediate values of $\kappa$, which lead to an intermediate size of $|\AF|$. This result corroborates our theoretical bound. As expected, increasing the dataset size reduces the planning loss. Most importantly, we see a bias-variance trade off with the variation in the size of the affordance, as predicted by the bound in Sec.~\ref{sec:planningloss}. Models learned with scarce data and a selective affordance (high $\kappa$) lead to high errors, due to bias. If the affordance is not sufficiently selective, the model estimated is inaccurate, due to high variance, and the planning loss is also higher. 

\begin{figure}[t]
    \vskip -0.1in %
    \centering
    \includegraphics[width=0.8\columnwidth]{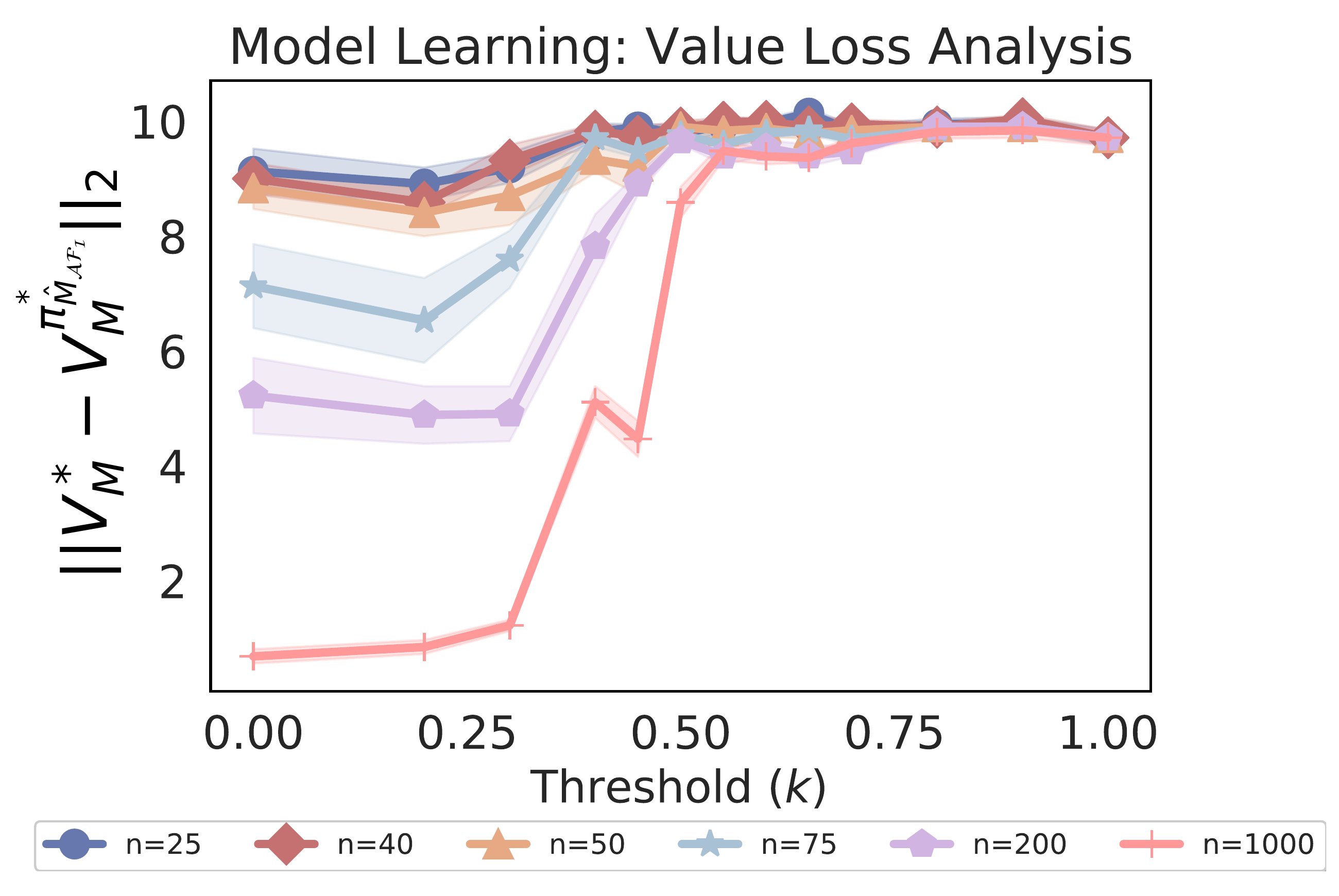}
    \caption{\label{fig:planning_value_loss}\textbf{Planning Value Loss Evaluation.} A model $\hat{M}_{\AF}$ is learned from the data, but state-action pairs that are not in the affordance are excluded.
    In the small data regime, an intermediate value of $\kappa$ is optimal as anticipated. With increase in the amount of data used to estimate the model, the planning loss eventually shrinks, as predicted by the theory, and increasing $|\AF|$ becomes better. The shaded areas show the standard error of the mean over 10 independent runs.}
    \vskip -0.2in %
\end{figure}

\section{Learning Affordances and Partial-Models with Function Approximation}\label{sec:experiments_learning}

In the previous section, we investigated the use of affordance-aware models to improve planning. However, the affordance was given a priori. In this section, we describe how we can learn and leverage  affordances in large state and action spaces, by using function approximation. In Sec.~\ref{sec:learning_aff}, we describe how to learn affordances through experience collected by an agent, in both discrete and continuous environments. In Sec.~\ref{sec:model_learning}, we then use the learned affordances to estimate partial models, which offer simplicity and better generalization.

\subsection{Learning Affordances}
\label{sec:learning_aff}

We represent an affordance as a classifier, $A_\theta(s,a,I)$, parameterized by $\theta$, which predicts whether 
a state $s\in{\cal S}$ and action $a\in{\cal A}$ can complete an intent, $I\in\mathcal{I}$. We assume that we have access to an intent-completion function, $c(s,a,s',I): {\cal S}\times{\cal A}\times{\cal S}\times \mathcal{I} \rightarrow [0, 1]$, that indicates if $s'\in I_a(s)$ for a given intent\footnote{In this work we consider intents that can be completed in one step (for example, changing state), so will always be 0 or 1. However, there are no restrictions on using multi-step intents, by using discounting, or using learned intents.}.  Transitions $(s,a,s')$, are collected from the environment and their intent completions $c(s,a,s',I)~\forall I \in \mathcal{I}$ are evaluated, to create a dataset, $\mathcal{D}$. We use the standard cross-entropy objective to train $A_\theta$: 
\begin{equation*}    \label{eqn:affordance_loss}
\mathcal{O}_A(\theta)=-\sum_{(s,a,s')\in\mathcal{D}}\sum_{I\in \mathcal{I}}c(s,a,s',I)\log A_\theta(s,a,I)
\end{equation*}

\begin{figure}[t!]
    \vskip -0.1in
    \begin{center}
    \includegraphics[trim={0 0.3cm 0 0},clip,width=0.4 \textwidth]{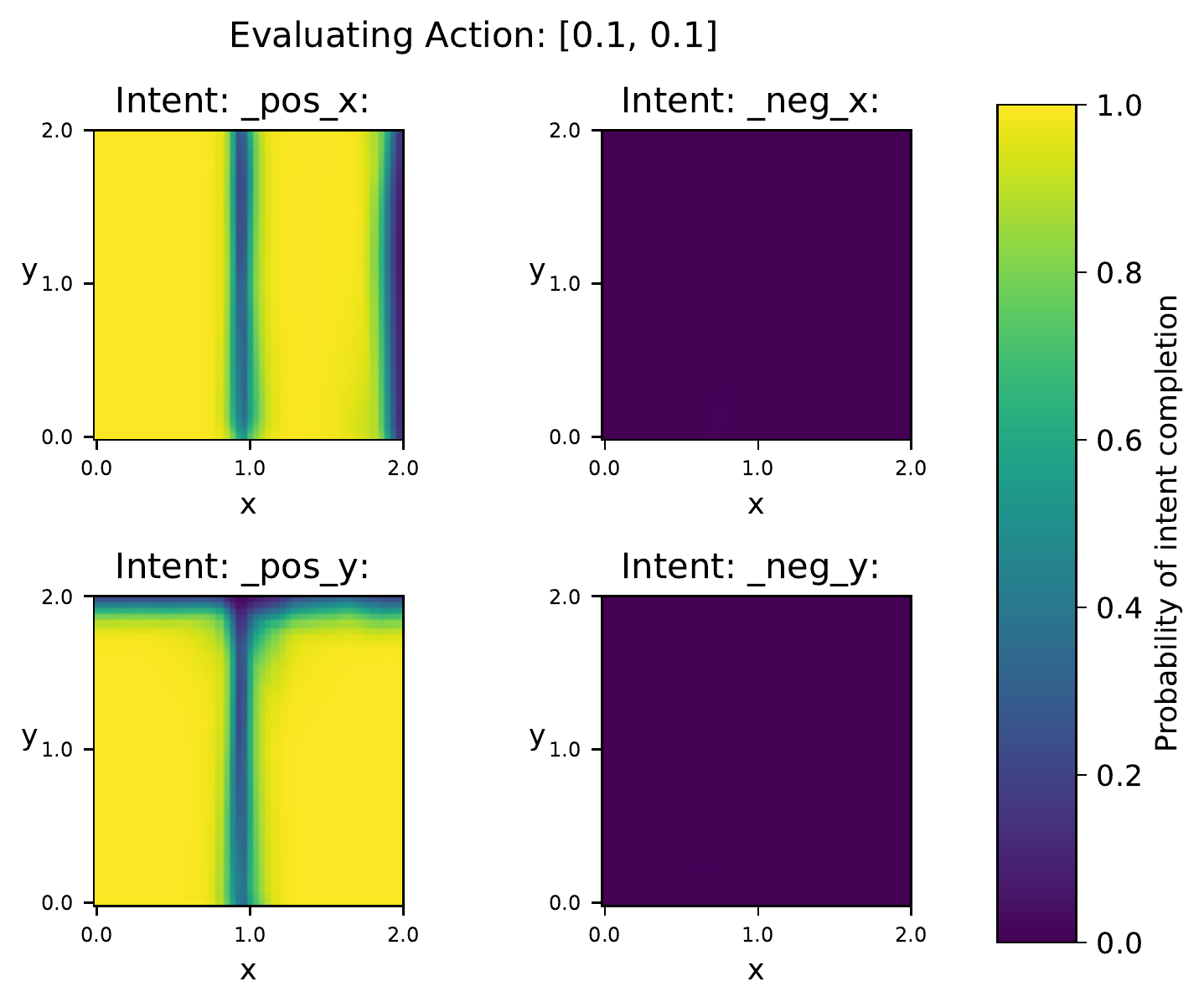}
    \vskip -0.1in
    \caption{\label{fig:continuous_intents}\textbf{Learned affordances in a continuous world.} Heatmaps show the probability of completing four distinct intents by the affordance classifier, $A_\theta$ for the the action, $F=(0.1,0.1)$ at every position in the world. $A_\theta$ correctly predicts that the intent $+\Delta x$ cannot be completed near the right walls, since going right has no effect on the agent's position there. Similarly $+\Delta y$ cannot be completed near the upper walls. Finally the two intents, $-\Delta y$ and $-\Delta x$, which describe movement in the opposite direction of the action, have close to zero probability.}
    \end{center}
    \vskip -0.2in %
\end{figure}

\emph{Experimental Setup:} We consider a continuous world where the agent can be in any position $(x, y)$ within a $2\times2$ 2D box (Fig~\ref{fig:continuousgridworld}). There is an impassable wall that divides the environment in two. 
When the environment resets, the starting position of the agent drifts from one side of the wall to the other. The action space consists of two displacements for each direction, $F=(F_x, F_y)$. The new position of the agent is drawn from $\mathcal{N}(\mu=(x+F_x, y+F_y), \sigma=0.1)$. If the action causes the movement through a wall, the position remains unchanged. We describe intents as movement in a particular direction: given state transition $(x, x')$,  $c(x,F,x',+\Delta x)=1, \forall F$ iff $x' - x > \delta$, for $\delta\in\mathbb{R}$, and 0 otherwise. We consider four intents: $\mathcal{I}=\{+\Delta x,-\Delta x,+\Delta y,-\Delta y\}$. Training data is collected by taking uniformly random actions with maximum displacement magnitude of $0.5$. We use a two layer neural network with 32 hidden units and RELU non-linearities \cite{nair2010rectified} and the Adam optimizer \cite{kingma2014adam} with learning rate $0.1$ to learn $A_\theta$.

\begin{figure*}
    \vskip -0.1in %
    \begin{center}
    \includegraphics[trim={2.5cm, 1.3cm, 0.5cm, 1.0cm},clip,width=1.1\textwidth]{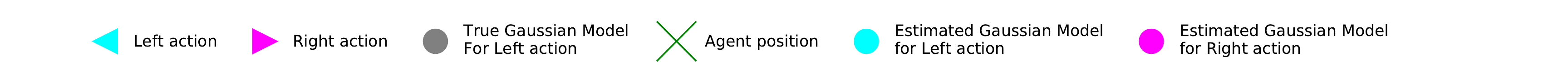}\\
    \vskip -0.1in
    \subfigure[{In-distribution $F_x=0.2$, Baseline Model}]{
    \includegraphics[width=0.24\textwidth]{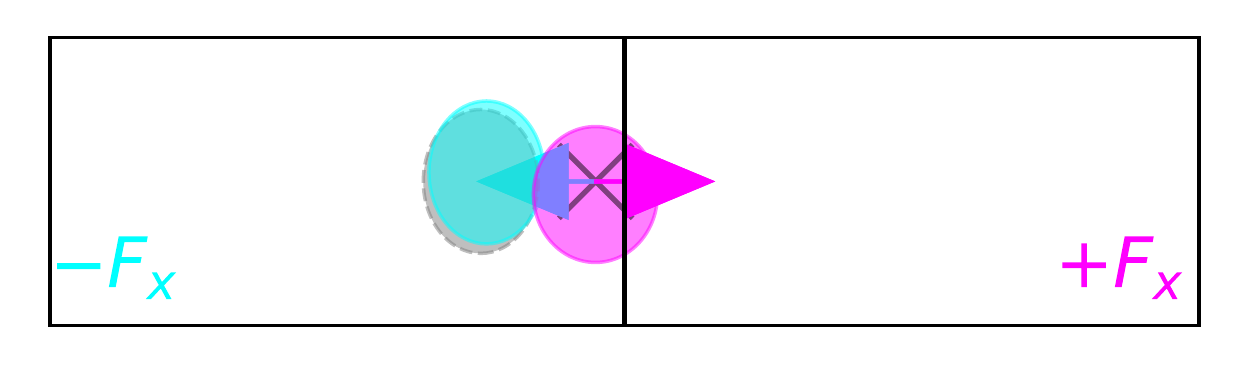}
    \includegraphics[width=0.24\textwidth]{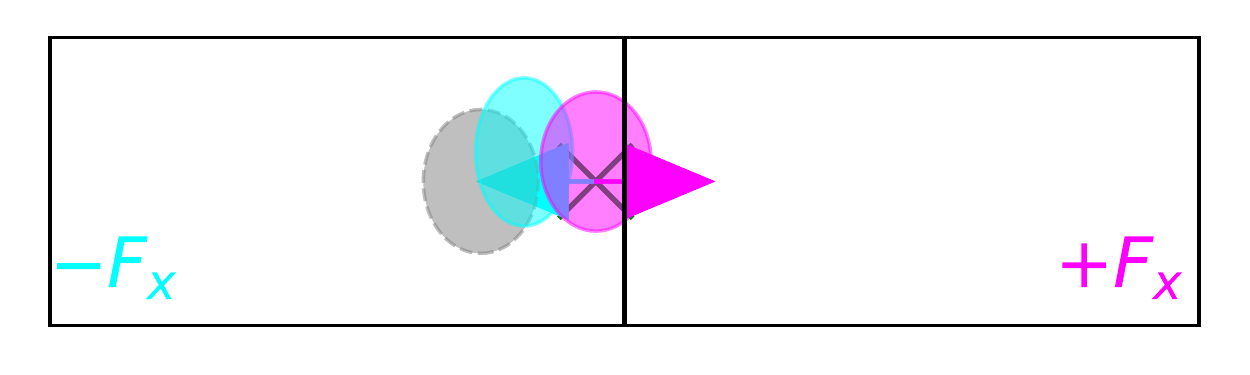}
    \label{fig:eval_vanilla_model}
    }
    \subfigure[{Out-of-distribution $F_x=0.75$, Baseline Models}]{
    \includegraphics[width=0.24\textwidth]{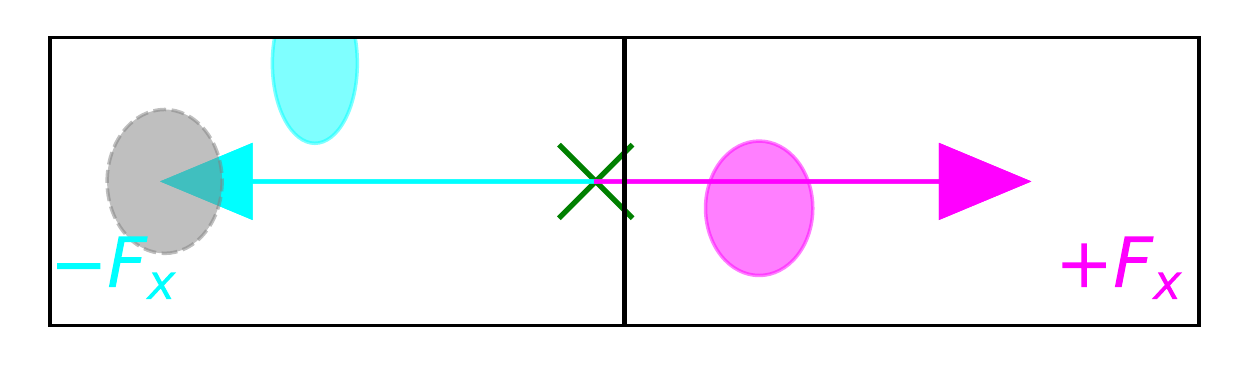}
    \includegraphics[width=0.24\textwidth]{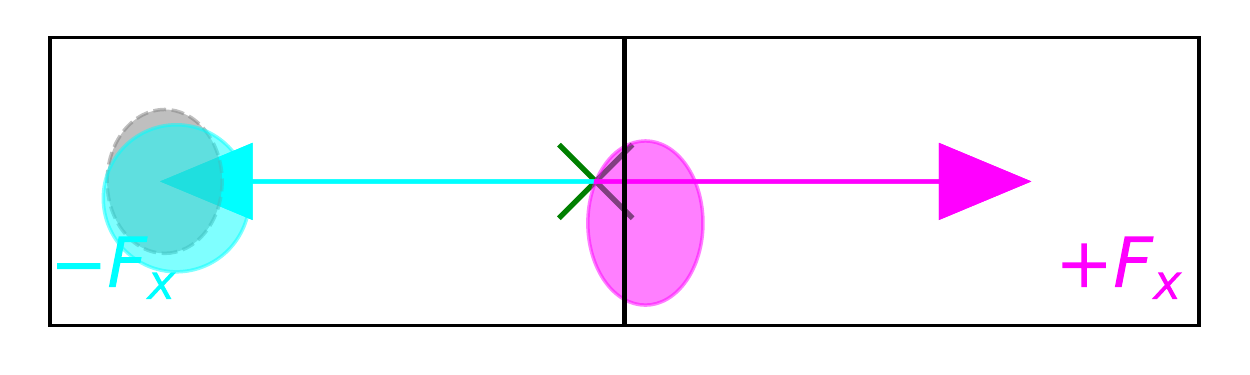}
    \label{fig:generalization_vanilla_model}
    }
    \subfigure[{In-distribution $F_x=0.2$, Affordance-aware Model}]{
    \includegraphics[width=0.24\textwidth]{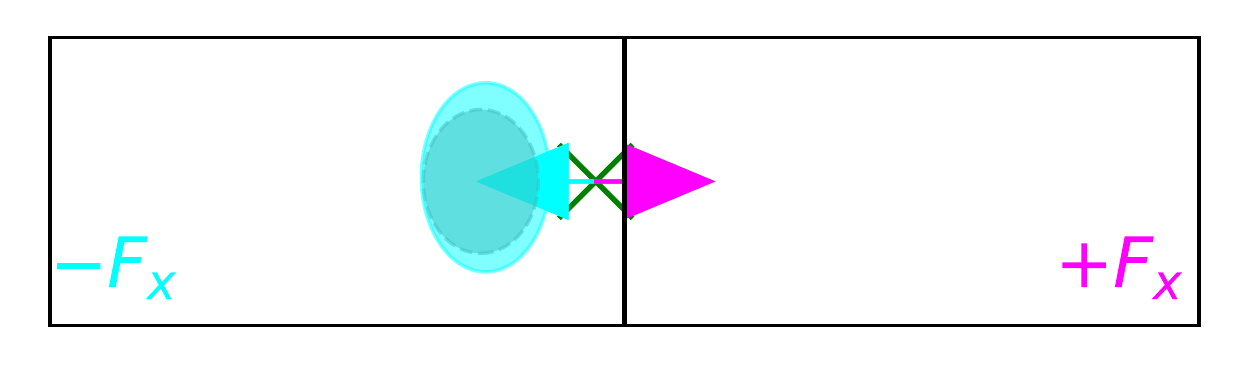}
    \includegraphics[width=0.24\textwidth]{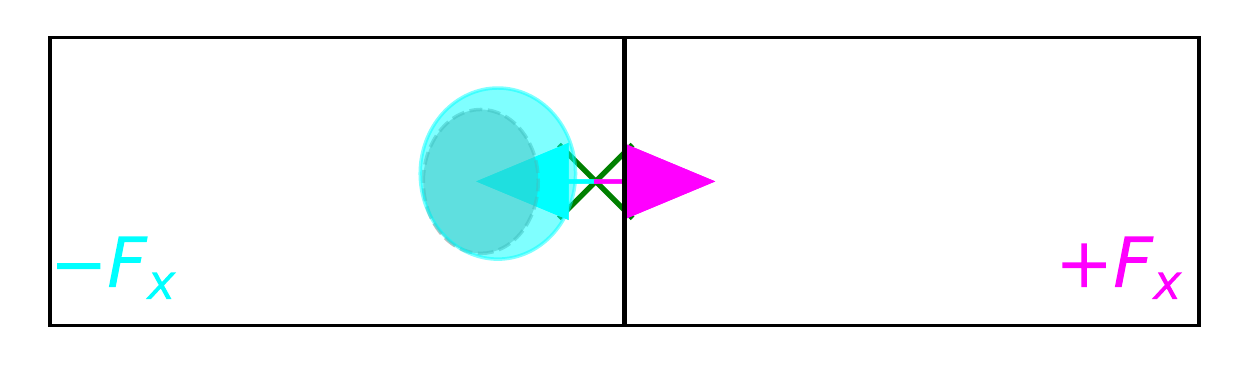}
    \label{fig:eval_affordance_model}
    }
    \subfigure[{Out-of-distribution $F_x=0.75$, Affordance-aware Model}]{
    \includegraphics[width=0.24\textwidth]{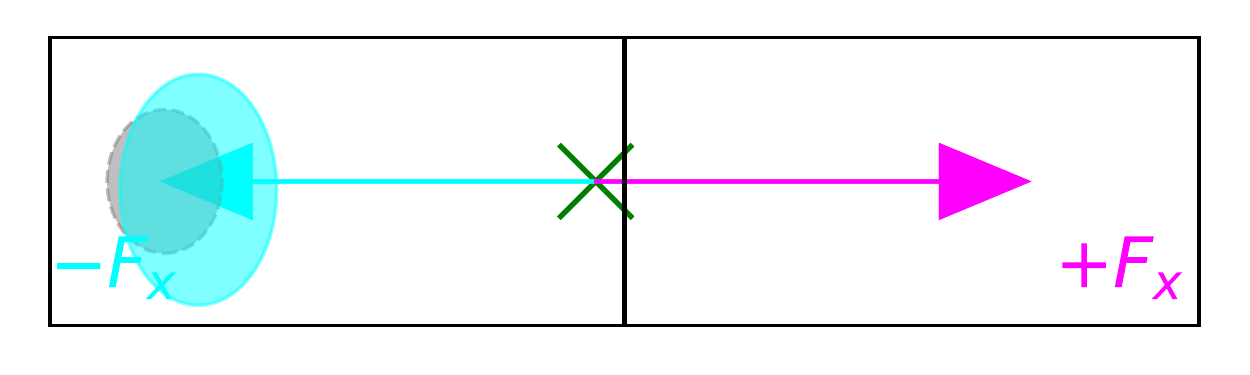}
    \includegraphics[width=0.24\textwidth]{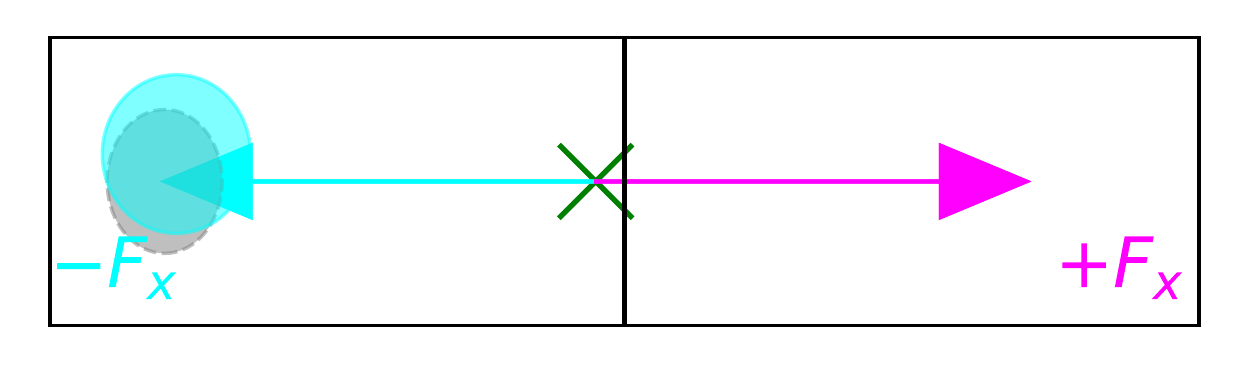}%
    \label{fig:generalization_affordance_model}}
    \vskip -0.1in
    \caption{\label{fig:eval:indist}\textbf{Evaluation of trained baseline and affordance-aware models for two independent seeds.} Top row (a,b) features the Baseline model while the bottom row (c, d) features the Affordance-aware Model.
     (a) On in-distribution actions, the baseline transition model learns reasonable predictions. For $+F_x$, the agent spreads its mass around its position, which might lead to predictions orthogonal or in the backward direction.
     (b) For out-of-distribution actions, the baseline transition model shows two failure modes for $+F_x$:  It either predicts that the agent can go through the wall, or tries to distribute the mass of the next state along the wall. For the left action, $-F_x$, the model predicts a degenerate distribution showing orthogonal movement to the action.
     (c, d) Transition models trained with  affordances predict a reasonable transition distribution for both in-distribution and out of distribution actions. For $+F_x$, the affordance classifier, $A_\theta$ predicts that no intent can be completed, and therefore the transition model is never queried. 
    \small{The radii of the circles show 2 standard deviations of the predicted model. See Fig~\ref{fig:appendix_eval_vanilla_model}, \ref{fig:appendix_eval_affordance_model}, \ref{fig:appendix_generalization_vanilla_model}, and \ref{appendix_fig:generalization_affordance_model} for more comparisons.}
    }
    \end{center}
    \vskip -0.2in %
\end{figure*}

After $2000$ updates, we evaluate the probability of completing the four intents for action $F=(0.1, 0.1)$ %
The classifier correctly learns that it cannot complete any intent near the walls (Fig~\ref{fig:continuous_intents}), and that applying a positive $F_x$ and $F_y$ cannot complete the intents  $-\Delta x$ and $- \Delta y$.

Analogous to the learning process in continuous environments, we are able to learn affordances in the discrete gridworld (Fig~\ref{fig:appendix_Illustration}). The algorithm we use to learn affordances has no requirements on the state- action space and therefore is expected to scale to more complex environments provided we have access to intents.

\subsection{Affordance-aware Partial-Model Learning}
\label{sec:model_learning}
In this section with demonstrate the practical use of having an affordance classifier in combination with a generative model of the environment. In particular, during training we use the affordance classifier to mask out transition data which do not complete any intent allowing the model to focus on learning from transitions that are relevant. During inference, we first query the affordance classifier to output if a state-action pair achieves any intent before querying the model for predictions. We show that the affordance-aware partial model produces better qualitative predictions and generalizes to out-of-distribution transition data.

\textit{Experimental Setup:} We re-use the continuous world 
from Sec~\ref{sec:learning_aff}. Since the underlying world uses a Gaussian model for transition noise, we use a Gaussian generative model, $P_\phi(s'|a,s)=\mathcal{N}(\mu_\phi(s,a), \sigma_\phi(s,a))$, to estimate the transition dynamics. Here $\mu_\phi$ and $\sigma_\phi$ are function approximators with parameters $\phi$, that estimate the mean and standard deviation of the distribution. To obtain a baseline model, we maximize the log probability of the next state transition: $\mathcal{O}_\text{baseline}(\phi) = \sum_{(s,a,s')\in \mathcal{D}}\log P_\phi(s'|s,a)$. 
To train an affordance-aware model, we use the outputs of the affordance classifier, $A_\theta$, to mask the loss for $P_\phi$: 
\begin{equation*}
\small{\mathcal{O}_{\text{aff}}(\phi)= \sum_{(s,a,s')\in \mathcal{D}}\mathbbm{1}\Big[\max_{\forall I\in\mathcal{I}}A_\theta(s,a,I) > k\Big]\log P_\phi(s'|s,a)}
\end{equation*}
for a threshold $k=0.5$ and $\mathbbm{1}$ is the indicator function that returns 1 if the argument is True\footnote{The full algorithm is detailed in Alg.~\ref{alg:aff_model_learning} and source code is provided.}. This loss focuses the learning of the model on transitions that complete an intent. In this setting, both the affordance classifier and the partial model are trained simultaneously.

After training for 7000 updates, both the baseline and affordance-aware models achieve a similar training loss (Fig.~\ref{fig:appendix_eval_vanilla_model}). To understand how these models behave, we inspect their predictions near the wall, by  querying two actions: $-F_x$, which leads to leftwards movement and $+F_x$. which will not have any effect due to the impassable wall. We keep $F_y=0$ fixed. 

\textit{In-distribution qualitative behavior:} We first consider the actions that are seen during training. For $F_x=-0.2$, two out of five baseline models predict an incorrect or offset transition distribution (See two representative seeds in Fig.~\ref{fig:eval_vanilla_model} and all Fig.~\ref{fig:appendix_eval_vanilla_model}) compared to near-perfect predictions in all runs of the affordance-aware model (See two representative seeds in Fig.~\ref{fig:eval_affordance_model} and all in Fig.~\ref{fig:appendix_eval_affordance_model}). For the action $F_x=+0.2$, which moves into the wall, the baseline model distributes the mass of its predictions along the wall. This is reasonable considering, that the agent will not move and the model is restricted to produce a Gaussian prediction. On the other hand, the affordance-aware model first uses $A_\theta$ to determine that the action can not complete an intent in this situation and $P_\phi$ is never queried. In a planning setting \citep{schrittwieser2019mastering}, such a model could be used to reduce the number of actions considered and thereby reduce computational complexity.

\textit{Out-of-distribution qualitative behavior:} To analyze how these learned models generalize, we evaluate them using a displacement never seen during training. For the action, $F_x=-0.75$, only 1 out of 5 baseline models predicts a good solution. Most distributions are offset or have a wide standard deviation (See two representative seeds in Fig.~\ref{fig:generalization_vanilla_model} and all in Fig.~\ref{fig:appendix_generalization_vanilla_model}). In contrast, all affordance-aware models predict reasonable distributions (See two representative seeds in Fig.~\ref{fig:generalization_affordance_model} and all in Fig.~\ref{appendix_fig:generalization_affordance_model}). For $F_x=+0.75$, which cannot be executed in the environment, the baseline model predicts that the agent can move through the wall (Fig.~\ref{fig:generalization_vanilla_model}) while the affordance-mask determines that the partial model cannot be queried for this action (Fig.~\ref{fig:generalization_affordance_model}).

These results indicate that despite having similar quantitative losses, the baseline and affordance-aware models have qualitatively different behavior. In retrospect, this is not surprising given that they have vastly different learning goals. The baseline models need to take into account the edge cases to make good predictions in all situations. On the other hand, since the classifier prevents us from querying the affordance-aware model when the action cannot be executed, it need only learn the rule: $\mathcal{N}(\mu=F_x+x, \sigma=0.1)$.

This section used the affordance classifier to focus the training of transition models on actions that are relevant. 
Our results %
show that the affordance-aware model can generalize to out-of-distribution actions. We expect that in environments with pixel-based~\citep{kaiser2020model} or other structured observations the model will likely also generalize to novel states.

\section{Related Work}\label{sec:related}

Affordances have a rich history in a variety of fields such as robotics, psychology, ecology, and computer vision. This notion originated in psychology~\citet{gibson1977theory, heft1989affordances, chemero2003outline}, but our approach is more related to the use of goals and preconditions for actions in classical AI systems, such as STRIPS \cite{fikes1971strips}.

In AI, researchers have also studied \textit{object-affordances}, in order to map actions to objects~\cite{slocum2000further, fitzpatrick2003learning, lopes2007affordance}. \citet{montesano2008learning} presented a developmental approach to learning object affordances. %
Their approach focuses on robots that learn basic skills such as visual segmentation, color, and shape detection. Different modalities have been used to detect affordances, such as visuo-motor simulation~\cite{schenck2012detecting}, visual characteristics~\cite{song2015learning}, and text embeddings~\cite{fulda2017can}.

In the context of MDPs, \citet{abel2014toward} define affordances as propositional functions on states. In particular, affordances consist of a mapping $\langle p, g  \rangle \rightarrow \mathcal{A'}$ where $\mathcal{A'} \subset \mathcal{A}$ represents the relevant action-possibilities, $p$ is a predicate on states $\mathcal{S} \rightarrow \{ 0,1 \} $ representing the precondition for the affordance and $g$ is an ungrounded predicate on states representing a lifted goal description. 
An extension learns affordances in the context of goal-based priors~\cite{abel2015goal}: they learn probability distributions over the optimality of each action for a given state and goal. However, their  approach relies on Object Oriented-MDPs~\cite{diuk2008object}, which assume the existence of \textit{objects} and \textit{object class} descriptions. Our approach is more general, using MDPs of any kind, and does not assume any knowledge about existing object classes or their attributes. Additionally, their work assumes that the model of the OO-MDP is given to the agent upfront.

\citet{cruz2014improving} demonstrate the utility of affordances given as \textit{prior knowledge} to RL agents. Later, \citet{cruz2016training,cruz2018multi} \textit{learned} contextual affordances as a tuple of $\langle \text{state}, \text{object}, \text{action}, \text{effect}\rangle$. Their approach, however, depends on known objects, such as ``sponge" or ``cup", in the construction of a \textit{state}, which poses considerable restrictions. Instead, we pursue  learning general purpose affordances from data and do not make any such object-centric assumptions.

A related family of approaches involves learning when to prune actions \cite{even2003action, sherstov2005improving, rosman2012good, zahavy2018learn} in order to cope with large action-spaces. Recently, AlphaStar used a similar approach of action masks to prune the action-space derived from human data ~\cite{vinyals2019grandmaster}. 
Our goal is to not only learn what can be done in a given state, but also build simpler, more robust models from this understanding. Besides, our approach to affordances can be used also to characterize  states, not just to eliminate actions.

Much of the work on partial models \citep{oh2017value,amos2018learning,guo2018neural,gregor2019shaping} focuses on models that predict only some of the state variables, or that make latent-space predictions \cite{schrittwieser2019mastering}, rather than restricting the space of actions considered in a given state. Our work is complimentary to existing techniques for building partial models in that they could still leverage our approach to reduce the number of actions, thereby further reducing the computational complexity of planning.

\section{Discussion and Future work}
We have laid the foundation of using affordances in RL to achieve two goals: 1) decreasing the computational complexity of planning, and 2) enabling the stable learning of partial models from data, which can generalize better than full models. %

One limitation of our work is that affordances are constructed based on intents, which are specified a priori. However, intents could also be learned from data---a problem which ties closely to subgoal discovery in temporal abstraction~\cite{mcgovern2001automatic,csimcsek2005identifying}. Critical states~\cite{goyal2019infobot, nair2019hierarchical} that are important for decision making would be ideal candidates to include in intent distributions.

A promising future theoretical direction would be to establish a bound on the number of samples required to learn affordance-aware partial models. Intuitively, and as shown in our results, it should be much faster to learn a simpler model on a subset of states and actions, compared to a complex model class for the full state and action space. Our results from Sec.~\ref{sec:experiments_learning} suggest that learned affordances can be used to estimate approximate transition models that are simpler and generalize better, despite reaching similar training losses. 

Finally, while we focused on affordances for primitive actions, we believe that an important application of this idea will be in the context of hierarchical reinforcement learning \citep{asadi2007effective, manoury2019hierarchical}, where intents are akin to subgoals and affordances could then take the role of initiation sets for options \cite{khetarpal2020options}. Intents could then be learned using information theoretic criteria proposed for option terminations as in \citet{harutyunyan2019termination}. Future avenues for this work in the context of option models~\cite{sutton1999between} could potentially include modelling long-term side effects of actions that do not match any of the intents. Moreover, if agents are allowed to create new, extended actions, using affordances would provide an effective way to control planning complexity, given ever-expanding action sets.

\section*{Acknowledgements}
The authors would like to thank Feryal Behbahani, Ankit Anand, Anita Gergely and the anonymous ICML reviewers for valuable feedback on a draft of this paper, Marc G. Bellemare, Shibl Mourad, Andre Barreto and the entire DeepMind Montreal team for constructive feedback throughout this project, Prakash Panangaden, Emmanuel Bengio, Kushal Arora and other colleagues at Mila for useful discussions. A special thank you to Kory Mathewson for timely and incredibly detailed feedback on an early version of this draft.

\bibliography{references}
\bibliographystyle{icml2020}

\appendix
\onecolumn
\renewcommand\thefigure{\thesection\arabic{figure}}    
\setcounter{section}{0}
\setcounter{theorem}{0}
\section{Appendix}
\label{sec:appendix}
\subsection{Reproducibility}
We follow the reproducibility checklist by \cite{joelle2019} to ensure this research is reproducible. For all algorithms presented, we include a clear description of the algorithm and source code is included with these supplementary materials. For any theoretical claims, we include: a statement of the result, a clear explanation of any assumptions, and complete proofs of any claims. For all figures that present empirical results, we include: the empirical details of how the experiments were run, a clear definition of the specific measure or statistics used to report results, and a description of results with the standard error in all cases. All figures with the returns show the standard error across multiple independent random seeds. In the following section, we provide complete details of the computing requirements and dependencies for the code.

\subsection{Computing and Open source libraries.}
All experiments were conducted using free Google Colab instances\footnote{\href{https://colab.research.google.com/}{https://colab.research.google.com/}}. We used EasyMDP version 0.0.5 for the GridWorlds in Section 6. For the function approximators and probabilistic models in Section 7 we used Tensorflow version 2.1\footnote{\href{tensorflow.org}{tensorflow.org}} and Tensorflow Probability version 0.9.0\footnote{\href{https://github.com/tensorflow/probability/releases/tag/v0.9.0}{https://github.com/tensorflow/probability/releases/tag/v0.9.0}} respectively. Source code is provided at \href{https://tinyurl.com/y9xkheme}{https://tinyurl.com/y9xkheme}.

\subsection{Proofs}
\label{sec-proofs}

\subsubsection{Proof of Value Loss Bound}
\label{sec:appendix_valuelossbound}
\begin{theorem}
\label{thm:value_loss_analysis}
Let $\mathcal{I}$ be a set of intents and $\epsilon_{s,a}$ be the minimum degree to which an intent is satisfied for $(s,a)$.
\begin{equation*}
\sum_{s'} \Big| P_{\mathcal{I}}(s'|s,a) - P(s'|s,a) \Big| \leq \epsilon_{{s,a}}.
\end{equation*}
Let $\epsilon=\max_{s,a} \epsilon_{s,a}$. Then, the value loss between the optimal policy for the original MDP $M$ and the optimal policy $\pi^*_\mathcal{I}$ computed from  the induced MDP $M_{\cal I}$ is given by:
\begin{equation*}
    ||V^{\pi^{*}_\mathcal{I}}_{M} - V^*_M||_{\infty} \leq   2\epsilon \frac{\gamma \texttt{Rmax}}{(1-\gamma)^{2}},  %
\end{equation*}
where $\texttt{Rmax}$ is the maximum possible value of the reward.
\end{theorem}

\begin{proof}

The value loss  that we want to bound is given by:
\begin{equation}
    ||V^{\pi^{*}_\mathcal{I}}_{M} - V^*_M||_{\infty}= \max_{s \in \cal S} \Big| V^{\pi^{*}_\mathcal{I}}_{M}(s) - V^*_M(s) \Big|
 \leq \underbrace{\max_{s \in \cal S} \Big| V^*_{M}(s) - V^{*}_{M_{\cal I}}(s) \Big|}_\text{Term 1} + \underbrace{\max_{s \in \cal S} \Big| V^{\pi^{*}_\mathcal{I}}_{M}(s) - V^*_{M_{\cal I}}(s) \Big|}_\text{Term 2}
\end{equation}
where we used the triangle inequality in the last step.

In order to bound Term 1, we first define the distance function $d^{\mathrm{F}}_{M_1, M_2}$ between two MDPs $M_1$ and $M_2$ that differ only in their dynamics as follows.

\ddef{$d^{\mathrm{F}}_{M_1, M_2}$}{ \label{def:dfm}
Given two MDPs $M_1$ and $M_2$ with dynamics $P_1$ and $P_2$ respectively, and function $f: \cal S \to \mathbb{R}$, define:
\begin{equation*}
    d^{f}_{M_1, M_2}(s,a) := \Big | \mathbb{E}_{s' \sim P_1(.|s,a)} [f(s')] -  \mathbb{E}_{s' \sim P_2(.|s,a)} [f(s')]  \Big |.
\end{equation*}
For any set of functions $\mathrm{F}$, define:
\begin{equation}
 d^{\mathrm{F}}_{M_1, M_2}(s,a) := \sup_{f \in \mathrm{F}} d^{f}_{M_1, M_2}(s,a).
\end{equation}
}

We now introduce Lemma 1 from \cite{jiang2018pac}, which we will use to establish our result.
\begin{lemma}
\label{lemma:jiang}
\cite{jiang2018pac} Given any $M_1$ and $M_2$, and any set $\mathrm{F}$ of value functions containing $V^{*}_{M_1}$ (i.e. $V^{*}_{M_1}\in \mathcal{F}$), we have: \begin{equation}
    ||V^{*}_{M_1} - V^{*}_{M_2}||_{\infty} \leq H ||d^{\mathrm{F}}_{M_1, M_2}||_{\infty},
\end{equation} where H is the horizon used in computing the value functions $V^{*}_{M_1}$ and $V^{*}_{M_2}$.
\end{lemma}

We invoke Lemma~\ref{lemma:jiang} in the case of discounted infinite horizon MDPs. Therefore:
\begin{equation} \label{eq:value_loss_1}
    ||V^{*}_{M} - V^{*}_{M_{\cal I}}||_{\infty} \leq  \frac{1}{(1 - \gamma)} ||d^{\mathrm{F}}_{M, M_{\cal I}}||_{\infty}
\end{equation} 

Consider now that $M_1$ and $M_2$ have the same reward function, and let $f : S \to \mathbb{R}$ be a function bounded by $\texttt{Rmax}$. Using the Bellman equation,  the term $d^{f}_{M, M_{\cal I}}$ can be expanded as follows:
\begin{equation*}
\begin{split}
  d^{f}_{M, M_{\cal I}}(s,a) &=  \Big| \sum_{s'} P(s'|s,a) \gamma f(s')
  - \sum_{s'} P_{\cal I}(s'|s,a) \gamma f(s') \Big|  \\
   &\leq  \sum_{s'} \gamma |f(s')| \left| P(s'|s,a) -  P_{\cal I}(s'|s,a)   \right| \leq \epsilon_{{s,a}} \frac{\gamma \texttt{Rmax}}{(1-\gamma)} 
\end{split}
\end{equation*}
where in the last step we used the notation introduced in Eq.~(\ref{eq:epsilon}).

Now, to be able to plug this bound back in Equation~\ref{eq:value_loss_1}, we need to consider the infinity norm: 
\begin{align}
  ||d^{\mathcal{F}}_{M, M_{\cal I}}||_{\infty} & = \max_{s,a}  \epsilon_{{s,a}} \frac{\gamma \texttt{Rmax}}{(1-\gamma)} 
\end{align}
Plugging the above back in Equation~\ref{eq:value_loss_1} and using Equation~\ref{eq:epsilon} yields:
\begin{equation}
    ||V^{*}_{M} - V^{*}_{M_{\cal I}}||_{\infty} \leq \epsilon \frac{\gamma \texttt{Rmax}}{(1-\gamma)^{2}}.
\end{equation}

We now consider the Term 2. Note that we now have to analyze the value loss between the optimal policy for the original MDP $M$ and the optimal policy $\pi^*_\mathcal{I}$ for the induced MDP $M_{\cal I}$. In other words, we would like to bound the policy evaluation error in the intended MDP as follows:
\begin{equation}
    \max_{s \in \cal S} \Big| V^{\pi^{*}_\mathcal{I}}_{M}(s) - V^*_{M_{\cal I}}(s) \Big| = \max_{s \in \cal S} \Big| V^{\pi^{*}_\mathcal{I}}_{M}(s) - V^{\pi^{*}_\mathcal{I}}_{M_{\cal I}}(s) \Big|
\end{equation}

Expanding each term as follows:
\begin{equation}
   V^{\pi^{*}_\mathcal{I}}_{M}(s) - V^{\pi^{*}_\mathcal{I}}_{M_{\cal I}}(s) =  \Big( R(s, \pi^{*}_\mathcal{I}(s))  + \gamma \mathbb{E}_{s' \sim P(.|s,a)} [V^{\pi^{*}_\mathcal{I}}_{M}(s')] \Big) - \Big( R(s, \pi^{*}_\mathcal{I}(s))  + \gamma \mathbb{E}_{s' \sim P_{\cal I}(.|s,a)} [V^{\pi^{*}_\mathcal{I}}_{M_{\cal I}}(s')] \Big)
\end{equation}

Considering that the rewards are known and same, we have: 
\begin{equation}
   V^{\pi^{*}_\mathcal{I}}_{M}(s) - V^{\pi^{*}_\mathcal{I}}_{M_{\cal I}}(s) =  \gamma \mathbb{E}_{s' \sim P(.|s,a)} [V^{\pi^{*}_\mathcal{I}}_{M}(s')]  - \gamma \mathbb{E}_{s' \sim P_{\cal I}(.|s,a)} [V^{\pi^{*}_\mathcal{I}}_{M_{\cal I}}(s')] 
\end{equation}

Considering the max over all states as we are interested in Term 1, and following through the proof of Lemma~\ref{lemma:jiang}, we have:
\begin{equation}
    \max_{s \in \cal S} \Big| V^{\pi^{*}_\mathcal{I}}_{M}(s) - V^*_{M_{\cal I}}(s) \Big| \leq ||d^{\mathcal{F}}_{M, M_{\cal I}}||_{\infty} + \max_{\cal S , \cal A} \Big| \mathbb{E}_{s' \sim P_{\cal I}(.|s,a)} [V^{\pi^{*}_\mathcal{I}}_{M}(s')] - \mathbb{E}_{s' \sim P_{\cal I}(.|s,a)} [V^{\pi^{*}_\mathcal{I}}_{M_{\cal I}}(s')]\Big|
\end{equation}

Using the notation introduced in Eq.~(\ref{eq:epsilon}) and expanding the inequality $H=1/(1-\gamma)$ times, it follows:
\begin{equation}
    \max_{s \in \cal S} \Big| V^{\pi^{*}_\mathcal{I}}_{M}(s) - V^*_{M_{\cal I}}(s) \Big| \leq \epsilon \frac{\gamma \texttt{Rmax}}{(1-\gamma)^{2}}
\end{equation}

Combining Term 1 and 2 bounds yields the final result in Equation~\ref{eq:theorem1}.
\end{proof}

\subsubsection{Proof of Planning Loss Bound}
\label{sec:appendix_planninglossbound}
\ddef{Policy class $\Pi_{\mathcal{I}}$}{Given affordance $ \AF$, let  $\mathcal{M_I}$ be the set of MDPs over the state-action pairs in $\AF$, and let 
\begin{equation*}
\Pi_{\mathcal{I}} = \{ \pi^{*}_M \} \cup \{ \pi : \exists \bar M \in  \mathcal{M_I} \text{ s.t. }  \pi \text{ is optimal in } \bar M  \}.
\end{equation*}}

\begin{theorem}
Let $\hat{M}_\AF$ be the approximate MDP over affordable state-action pairs. Then the certainty equivalence planning with $\hat{M}_\AF$ has planning loss 
\begin{equation*}
\Big|\Big|V^*_M - V^{\pi^{*}_{\hat{M}_\AF}}_M \Big|\Big|_{\infty} \leq \frac{2 \texttt{Rmax}}{(1-\gamma)^2} \left( 2\gamma \epsilon +  \sqrt{\frac{1}{2n} \log \frac{2 |\AF| |\Pi_\mathcal{I}|}{\delta}} \right)
\end{equation*} 
with probability at least $1-\delta$.
\label{theorem:planningvaluelossbound}
\end{theorem}

\begin{proof}
Let us consider that the world is represented by an MDP $M: \langle {\cal S}, {\cal A}, R, \gamma, P \rangle $. Let $P_{\cal I}$ to denote proxy models based on the collection of intents $\mathcal{I}$, and the resulting \emph{intended} MDP is denoted by $M_{\cal I}: \langle {\cal S}, {\cal A}, r, P_{\cal I}, \gamma \rangle$. We are now interested in estimating $M_{\cal I}$ via the data samples experienced by the agent. Let's consider this estimated model to be $\hat{M}_\AF$, where $\hat{M}_\AF: \langle {\cal S}, {\cal A}, R, \gamma, \hat{P_{\cal I}} \rangle $. In particular, we are interested in the CE-control policy, which is discussed in more detail in the main paper. Let $\pi^*_{\hat{M}_\AF}$ be the optimal policy of MDP 
 $\hat{M}_\AF$. We quantify the largest absolute difference (over states) between the value of the true optimal policy with respect to the true model, $\pi^*_{M}$ and that of $\pi^*_{\hat{M}_\AF}$ when evaluated in $M$:
\begin{equation}
    \text{\textbf{Planning Value Loss: }} \Big| \Big|  V^*_{M} - V^{\pi^*_{\hat{M}_\AF}}_M \Big| \Big|_{\infty}
\end{equation}

It is to be noted that our proof builds on the theory proposed by \citet{jiang2015dependence}. However, we are not concerned with the dependence of planning value loss on the effective horizon, and consider a fixed discount factor $\gamma$. We follow through the steps of the proof of Theorem 2 of \citet{jiang2015dependence} and prove this theorem using the lemmas below: Lemma~\ref{lemma2}, Lemma~\ref{lemma3}, and Lemma~\ref{lemma4}, and \ref{lemma:jiangv2}.

\lemma{\cite{jiang2015dependence}}{ For any MDP $\hat{M}_\AF$ with $\hat{R}=R$,  which is an approximate model of the MDP given by the intent collection $\mathcal{I}$, we have
\begin{equation}
    \Big|\Big| V^*_{M_{\cal I}} - V^{\pi^*_{\hat{M}_\AF}}_{M_{\cal I}} \Big|\Big|_{\infty} \leq  2 \max_{\pi \in \Pi_\mathcal{I}} ||V^{\pi}_{M_{\cal I}} - V^{\pi}_{\hat{M}_\AF}||_{\infty} 
\end{equation}
\label{lemma2}
}
\begin{proof}
$\forall s \in S$
Let us consider:
\begin{align*}
    &  V^*_{M_{\cal I}}(s) - V^{\pi^*_{\hat{M}_\AF}}_{M_{\cal I}}(s) \\
    &= \Big( V^*_{M_{\cal I}}(s) - V^{\pi^*_{M_{\cal I}}}_{\hat{M}_\AF}(s) \Big) + \underbrace{\Big( V^{\pi^*_{M_{\cal I}}}_{\hat{M}_\AF}(s) - V^{*}_{\hat{M}_\AF}(s) \Big)}_{\leq 0} + \Big( V^{*}_{\hat{M}_\AF}(s) - V^{\pi^*_{\hat{M}_\AF}}_{M_{\cal I}}(s) \Big)\\
    &\leq \Big( V^*_{M_{\cal I}}(s) - V^{\pi^*_{M_{\cal I}}}_{\hat{M}_\AF}(s) \Big) - \Big( V^{*}_{\hat{M}_\AF}(s) - V^{\pi^*_{\hat{M}_\AF}}_{M_{\cal I}}(s) \Big) \\
    &\leq 2 \max_{\pi \in \Big\{ \pi^{*}_{\hat{M}_\AF}, \pi^{*}_{M_{\cal I}}  \Big\} } \Big| V^{\pi}_{M_{\cal I}}(s)  - V^{\pi}_{\hat{M}_\AF}(s) \Big|
\end{align*}
Taking a max over all states on both sides of the inequality and noticing that the set of all policies is a trivial super set of $\Big\{ \pi^{*}_{\hat{M}_\AF}, \pi^{*}_{M_{\cal I}}  \Big\} \in \Pi_\mathcal{I}$, from which the final result follows.
\end{proof}

We now turn to Lemma~\ref{lemma3}.
\lemma{\cite{jiang2015dependence}}{ For any MDP $\hat{M}_\AF$ with $\hat{R}=R$ bounded by $[0, R_{max}]$ which is an approximate of the MDP  estimated from data experienced in the world for a set of intents ${\cal I}$,
\begin{equation}
    \Big|\Big| V^{\pi}_{M_{\cal I}} - V^{\pi}_{\hat{M}_\AF}\Big|\Big|_{\infty} \leq  \frac{1}{(1-\gamma)} \left\|  \sum_{a \in \AF(s)} \Big( R(s,a) + \gamma \langle \hat{P_{\cal I}}(s,a,;), V^{\pi}_{M_{\cal I}} \rangle \Big) -  V^{\pi}_{M_{\cal I}}  \right\| _{\infty} .
\end{equation}
\label{lemma3}
}
\begin{proof}

Given any policy $\pi$, define state-value function $V_0, V_1, \dots V_m$ such that $V_0 = V^{\pi}_{M_{\cal I}}$,

From this point onward, we use $\AF(a)$ and $\AF(s)$ to denote affordable states and affordable actions respectively.

$\forall s \in \AF(a)$
\begin{equation*}
    V_m(s) = \sum_{a \in \AF(s)} \pi(a|s) \Big( R(s,a) + \gamma \langle \hat{P_{\cal I}}(s,a,;), V_{m-1} \rangle \Big)
\end{equation*}

Therefore:
\begin{align}
\label{eq:intermediatesteps}
\begin{split}
    ||V_m - V_{m-1}||_{\infty} &= \max_s \left[  \sum_{a \in \AF(s)} \pi(a|s) \gamma \left\langle \hat{P_{\cal I}}(s,a,;), (V_{m-1} - V_{m-2}) \right\rangle  \right] \\
    &\leq \gamma \max_s \sum_{a \in \AF(s)} \pi(a|s) \hat{P_{\cal I}}(s,a,;) ||V_{m-1} - V_{m-2}||_\infty
\end{split}
\end{align}

Since $\langle \hat{P_{\cal I}}(s,a,;), f \rangle = \sum_{s'} \hat{P_{\cal I}}(s,a,s') \cdot f(s')$
\begin{align*}
\begin{split}
    \sum_{s^{'}} \hat{P_{\cal I}}(s,a,s') \cdot f(s') &\leq \sum_{s'} \hat{P_{\cal I}}(s,a,s') \cdot |f|_\infty \\
    &= |f|_\infty \text{ since } ||\hat{P_{\cal I}}||_{1} = 1
\end{split}
\end{align*}

Therefore 
\begin{align*}
    ||V_m - V_{m-1}||_{\infty} \leq \gamma || V_{m-1} - V_{m-2} || _\infty
\end{align*}

Therefore, $||V_m - V_0||_\infty \sum_{k=0}^{m-1} ||V_{k+1} - V_k||_\infty \leq  ||V_1 - V_0||_\infty \sum_{k=1}^{m-1} \gamma^{k-1}$

Taking the limit $m \rightarrow \infty $, $V_m \rightarrow V_{\hat{M}_\AF}^{\pi}$, and we have:
\begin{equation*}
    ||V_{\hat{M}_\AF} - V_0||_{\infty} \leq \frac{1}{1-\gamma} ||V_1 - V_0||_\infty
\end{equation*}
where notice that $V_0 = V^{\pi}_{M_{\cal I}}$ and $V_1 = \sum_{a \in \AF(s)} \pi(a|s) \Big( R + \gamma \langle \hat{P_{\cal I}}(s,a;), V_{M}^{\pi} \rangle \Big)$

Therefore,  
\begin{equation*}
    ||V^{\pi}_{M_{\cal I}} - V^{\pi}_{\hat{M}_\AF}||_{\infty} \leq  \frac{1}{(1-\gamma)} || \sum_{a \in \AF(s)} (R(s,a) + \gamma \langle \hat{P_{\cal I}}(s,a,;), V^{\pi}_{M_{\cal I}} \rangle) -  V^{\pi}_{M_{\cal I}}  ||_{\infty} 
\end{equation*}
\end{proof}

\begin{lemma} \label{lemma4}
The following holds with probability at least $1-\delta$:
\begin{equation*}
\Big|\Big|V^*_{M_{\cal I}} - V^{\pi^{*}_{\hat{M}_\AF}}_{M_{\cal I}} \Big|\Big|_{\infty} \leq \frac{2 \texttt{Rmax}}{(1-\gamma)^2} \sqrt{\frac{1}{2n} \log \frac{2 | \mathcal{\AF}| |\Pi_{\mathcal{I}}|}{\delta}}
\end{equation*} 
\end{lemma}

\begin{proof}

Using Lemma \ref{lemma2} and \ref{lemma3}, we have
\begin{equation*}
\begin{split}
    \left\| V^{\pi^*_{M_{\cal I}}}_{M_{\cal I}} - V^{\pi^*_{\hat{M}_\AF}}_{M_{\cal I}} \right\|_{\infty} &\leq  2 \max_{\pi \in \Pi_\mathcal{I}} ||V^{\pi}_{M_{\cal I}} - V^{\pi}_{\hat{M}_\AF}||_{\infty} \\
    &\leq \max_{\pi \in \Pi_\mathcal{I}}  \frac{2}{(1-\gamma)} \left\| \sum_{a \in \AF(s)} (R(s,a) + \gamma \langle \hat{P_{\cal I}}(s,a,;), V^{\pi}_{M_{\cal I}} \rangle) -  V^{\pi}_{M_{\cal I}}  \right\|_{\infty} \\
    &\leq  \max_{s \in S, \pi \in \Pi_\mathcal{I}} \frac{2}{(1-\gamma)} \left\| \sum_{a \in \AF(s)} (R(s,a) + \gamma \langle \hat{P_{\cal I}}(s,a,;), V^{\pi}_{M_{\cal I}} \rangle) -  V^{\pi}_{M_{\cal I}}  \right\|_{\infty}
\end{split}
\end{equation*}

Since $\sum_{a \in \AF(s)} (R(s,a) + \gamma \langle \hat{P_{\cal I}}(s,a,;), V^{\pi}_{M_{\cal I}} \rangle)$ is the average of the IID samples the agent obtains by interacting with the environment, bounded in $[0, \texttt{Rmax}]$ with mean $V^{\pi}_{M_{\cal I}}$ (for any $s, a, \pi$ tuple). Then according to Hoeffdings inequality,

\begin{equation*}
    \forall t \geq 0, \; P\Big( \Big|\sum_{a \in \AF(s)}  R(s,a) + \gamma \langle \hat{P_{\cal I}}(s,a,;), V^{\pi}_{M_{\cal I}} \rangle - V^{\pi}_{M_{\cal I}} \Big| > t \Big) \leq 2 \exp \left\{ \frac{-2 n t^{2}}{(\texttt{Rmax})^{2} / (1-\gamma)^{2}} \right\}
\end{equation*}
To obtain a uniform bound over all $s, a, \pi$ tuples, we equate the RHS to $\frac{\delta}{|\AF(a)||\AF(s)|\Pi_\mathcal{I}|}$ and the result follows as shown below.

\begin{equation*}
\begin{split}
    2 \exp \left\{ \frac{-2 n t^{2}}{(\texttt{Rmax})^{2} / (1-\gamma)^{2}}\right\} &= \frac{\delta}{|\AF(a)||\AF(s)||\Pi_\mathcal{I}|} \\
    \frac{-2 n t^{2}}{(\texttt{Rmax})^{2} / (1-\gamma)^{2}} &= \log \frac{\delta}{2|\AF(a)||\AF(s)||\Pi_\mathcal{I}|} \\
    \frac{2 n t^{2}}{(\texttt{Rmax})^{2} / (1-\gamma)^{2}} &= \log \frac{2|\AF(a)||\AF(s)||\Pi_\mathcal{I}|}{\delta} \\
    t^{2} &= \frac{(\texttt{Rmax})^{2}}{(1-\gamma)^{2}} \frac{1}{2n} \log \frac{2 |\AF(a)||\AF(s)||\Pi_\mathcal{I}|}{\delta} \\
    t &= \frac{\texttt{Rmax}}{(1-\gamma)} \sqrt{\frac{1}{2n} \log \frac{2 |\AF(a)||\AF(s)||\Pi_\mathcal{I}|}{\delta}}
\end{split}
\end{equation*}

We express the state-action pairs in affordances as the size of affordances. Formally, the size of affordances for a intent can be expressed as $|\AF|$. 
Plugging this back, we get the final result:
\begin{equation*}
    ||V^*_{M_{\cal I}} - V^{\pi^*_{\hat{M}_\AF}}_{M_{\cal I}}||_{\infty} \leq \frac{2 \texttt{Rmax}}{(1-\gamma)^2} \sqrt{\frac{1}{2n} \log \frac{2 |\AF| |\Pi_\mathcal{I}|}{\delta}}
\end{equation*} 

\end{proof}

The following lemma is very similar to a result on the error with respect to the optimal value function, as proved in \cite{jiang2018pac}.
\begin{lemma}
\label{lemma:jiangv2}
Given any policy $\pi$, and any set $\mathrm{F}$ of value functions containing $V^{\pi}_{M}$, we have 
\begin{equation}
    ||V^{\pi}_{M} - V^{\pi}_{M_{\cal I}}||_{\infty} \leq \frac{\gamma}{1-\gamma} ||d^{\mathrm{F}}_{M, M_{\cal I}}||_{\infty}.
\end{equation}
\end{lemma}

\begin{proof}

In the proof below, for any model $M$, we will use ${\cal T}^\pi_M$ to denote the Bellman opearator 
$$ {\cal T}^\pi_M f = \sum_a \pi(a |s) \left( R(s,a) + \gamma \sum_{s'} P_M (s' |s,a) f(s')\right). $$
The Bellman operator has the following property (for any two models $M_2$ and $M_2$):
For the first term, 
\begin{align*}
&({\cal T}^\pi_{M_1} - {\cal T}^\pi_{M_2}) f(s) \\
& \quad = \sum_a \pi(a |s) \left( R(s,a) + \gamma \sum_{s'} P_{M_1} (s' |s,a) f(s')\right) - \sum_a \pi(a |s) \left( R(s,a) + \gamma \sum_{s'} P_{M_2} (s' |s,a) f(s')\right) \\
& \quad = \sum_a \pi(a |s) \gamma \left(  \sum_{s'} P_{M_1} (s' |s,a) f(s') - \sum_{s'} P_{M_2} (s' |s,a) f(s')\right) \\
& \quad \leq \sum_a \gamma  \pi(a |s) d^{f}_{M_1, M_{2}} =  \gamma d^{f}_{M_1, M_{2}}
\end{align*}
where $d^{f}_{M_1, M_{2}}$ is the metric defined in Section~\ref{def:dfm}.
\begin{align*}
&{\cal T}^\pi_{M} f_1(s) - {\cal T}^\pi_{M} f_2(s) = \\
& \quad = \sum_a \pi(a |s) \left( R(s,a) + \gamma \sum_{s'} P_{M} (s' |s,a) f_1(s')\right) - \sum_a \pi(a |s) \left( R(s,a) + \gamma \sum_{s'} P_{M} (s' |s,a) f_2(s')\right) \\
& \quad = \sum_a \pi(a |s) \gamma \sum_{s'} P_{M} (s' |s,a) \left( f_1(s') - f_2(s')\right)\\
& \quad \leq \sum_a \pi(a |s) \gamma \sum_{s'} P_{M} (s' |s,a) || f_1 - f_2||_\infty = \gamma || f_1 - f_2||_\infty
\end{align*}
Now, the following holds for the initial value error we are interested to bound:
\begin{align*}
& ||V^{\pi}_{M} - V^{\pi}_{M_{\cal I}}||_{\infty} \leq  ||V^{\pi}_{M} - {\cal T}^\pi_{M_{\cal I}} V^{\pi}_M||_{\infty} + ||{\cal T}^\pi_{M_{\cal I}} V^{\pi}_M - V^{\pi}_{M_{\cal I}}||_{\infty} \\ & \quad \quad = ||{\cal T}^\pi_{M} V^{\pi}_{M} - {\cal T}^\pi_{M_{\cal I}} V^{\pi}_M||_{\infty} + ||{\cal T}^\pi_{M_{\cal I}} V^{\pi}_M - {\cal T}^\pi_{M_{\cal I}} V^{\pi}_{M_{\cal I}}||_{\infty} \\
& \quad \quad = ||({\cal T}^\pi_{M} - {\cal T}^\pi_{M_{\cal I}}) V^{\pi}_M||_{\infty} + ||{\cal T}^\pi_{M_{\cal I}} (V^{\pi}_M  - V^{\pi}_{M_{\cal I}})||_{\infty} \\
& \quad \quad \leq \gamma || d^{\mathrm{F}}_{M, M_{\cal I}} ||_{\infty} + \gamma ||V^{\pi}_{M} - V^{\pi}_{M_{\cal I}} ||_{\infty}
\end{align*}
Unfolding the above to infinity, we obtain in the limit the following:
$$ ||V^{\pi}_{M} - V^{\pi}_{M_{\cal I}}||_{\infty} \leq \frac{\gamma}{1 - \gamma} || d^{\mathrm{F}}_{M, M_{\cal I}} ||_{\infty}$$
\end{proof}

Now that we have all the necessary Lemmas proved, we can use them to provide the bound for Theorem~\ref{theorem:planningvaluelossbound}:
\begin{equation*}
\Big| \Big|  V^*_{M} - V^{\pi^*_{\hat{M}_\AF}}_M \Big| \Big|_{\infty} \leq \Big| \Big|  V^*_{M} - V^{\pi^*_{M_{\cal I}}}_M \Big| \Big|_{\infty} + \Big| \Big|  V^{\pi^*_{M_{\cal I}}}_M - V^{*}_{M_{\cal I}} \Big| \Big|_{\infty} + \Big| \Big| V^{*}_{M_{\cal I}} - V^{\pi^*_{\hat{M}_\AF}}_{M_{\cal I}} \Big| \Big|_{\infty} + \Big| \Big| V^{\pi^*_{\hat{M}_\AF}}_{M_{\cal I}} - V^{\pi^*_{\hat{M}_\AF}}_M \Big| \Big|_{\infty}
\end{equation*}

Theorem~\ref{thm:value_loss_analysis} applies to the first term, Lemma~\ref{lemma:jiangv2} to the second and forth term, and Lemma~\ref{lemma4} for the third term. Finally,

\begin{align*}
\Big| \Big|  V^*_{M} - V^{\pi^*_{\hat{M}_\AF}}_M \Big| \Big|_{\infty} & \leq 2 \epsilon \frac{\gamma \texttt{Rmax}}{(1-\gamma)^{2}} + 2 \epsilon \frac{\gamma \texttt{Rmax}}{(1-\gamma)^{2}} + \frac{2 \texttt{Rmax}}{(1-\gamma)^2} \sqrt{\frac{1}{2n} \log \frac{2 |\AF| |\Pi_\mathcal{I}|}{\delta}} \\
& = \frac{2 \texttt{Rmax}}{(1-\gamma)^2} \left( 2\gamma \epsilon +  \sqrt{\frac{1}{2n} \log \frac{2 |\AF| |\Pi_\mathcal{I}|}{\delta}} \right)
\end{align*}

\end{proof}

\subsection{Empirical Validation: Additional Details}
\label{sec:algo-discrete}
\subsubsection{Computationally building affordances}
\label{sec:build-affordances}
Consider $\AF \subseteq {\mathcal S} \times {\mathcal A}$ as the subset of state-action pairs that complete a collection of intents specified a priori. To control the size of affordances, $|\AF|$, we introduce a threshold $k$. A state-action pair is deemed affordable if the intent is completed i.e. $s' \in I_a(s)$ and the $P(s, a, s') \geq k$. For $k=0.0$, affordances include all state-action pairs that achieves the intent. With increasing threshold values, $\AF$ contains relatively smaller subset of state-action space, resulting in a reduced affordance size. Our final affordance set considers all intents i.e. $ \cup_{I \in \mathcal{I}} \AF \subset {\cal S} \times {\cal A}$. We present the pseudocode for the empirical analysis of planning value loss bound in Algorithm~\ref{alg:planningvaluelossbound}.
\begin{algorithm}[h]
   \caption{Pseudo code: Planning Value Loss Analysis}
   \label{alg:planningvaluelossbound}
\begin{algorithmic}
   \STATE {\bfseries Require:} Collection of Intents $\mathcal{I}$, where each intent $\forall{s} \in {\cal S}, I_a(s) \in Dist(s)$
   \STATE {\bfseries Input:} MDP $M:\langle {\cal S}, {\cal A}, r, P\rangle$, number of trajectories $n$, thresholds $k$
   
   \textbf{1. Affordances $\AF$:}
   \STATE  $\AF \leftarrow$ Computationally Build Affordances $(M, \mathcal{I}, k)$ //As explained in Sec~\ref{sec:build-affordances}.
   
   \textbf{2. Learn Affordance-aware Model $\hat{M}_{\AF}$:}
   \STATE Transition tuples $(s, a, s', r ) \leftarrow$ collect trajectories
   \STATE $\hat{P_{\cal I}} \leftarrow $ Count $(s, a, s', r )$ to estimate model
   
   \textbf{3. Planning:}
   \STATE $\pi^{*}_M$ $\leftarrow$ Value Iteration($P, R$)
   \STATE $\pi^{*}_{\hat{M}_{\AF}}$ $\leftarrow$ Value Iteration($\hat{P_{\cal I}}, R$)
   
   \textbf{4. Certainty-Equivalence Control Evaluation:}
   \STATE $V^{\pi^{*}_M}_M$ $\leftarrow$ Policy Evaluation($\pi^{*}_M, P, R, \gamma$)
   \STATE $V^{\pi^{*}_{\hat{M}_{\AF}}}_M$ $\leftarrow$ Policy Evaluation($\pi^{*}_{\hat{M}_{\AF}}, P, R, \gamma$)
   \STATE Planning Loss $\leftarrow \Big| \Big| V^*_M - V^{\pi^*_{\hat{M}_\AF}}_M \Big| \Big|_2$
   
   \text{Repeat steps 1-4 for different values of $k$, $n$, and average over $m$ independent seeds.}
\end{algorithmic}
\end{algorithm}

\begin{figure}[h]
    \begin{center}
    \subfigure[One-room, Move Intent]{\label{fig:appendix_intent_changestate}\includegraphics[width=0.2\textwidth]{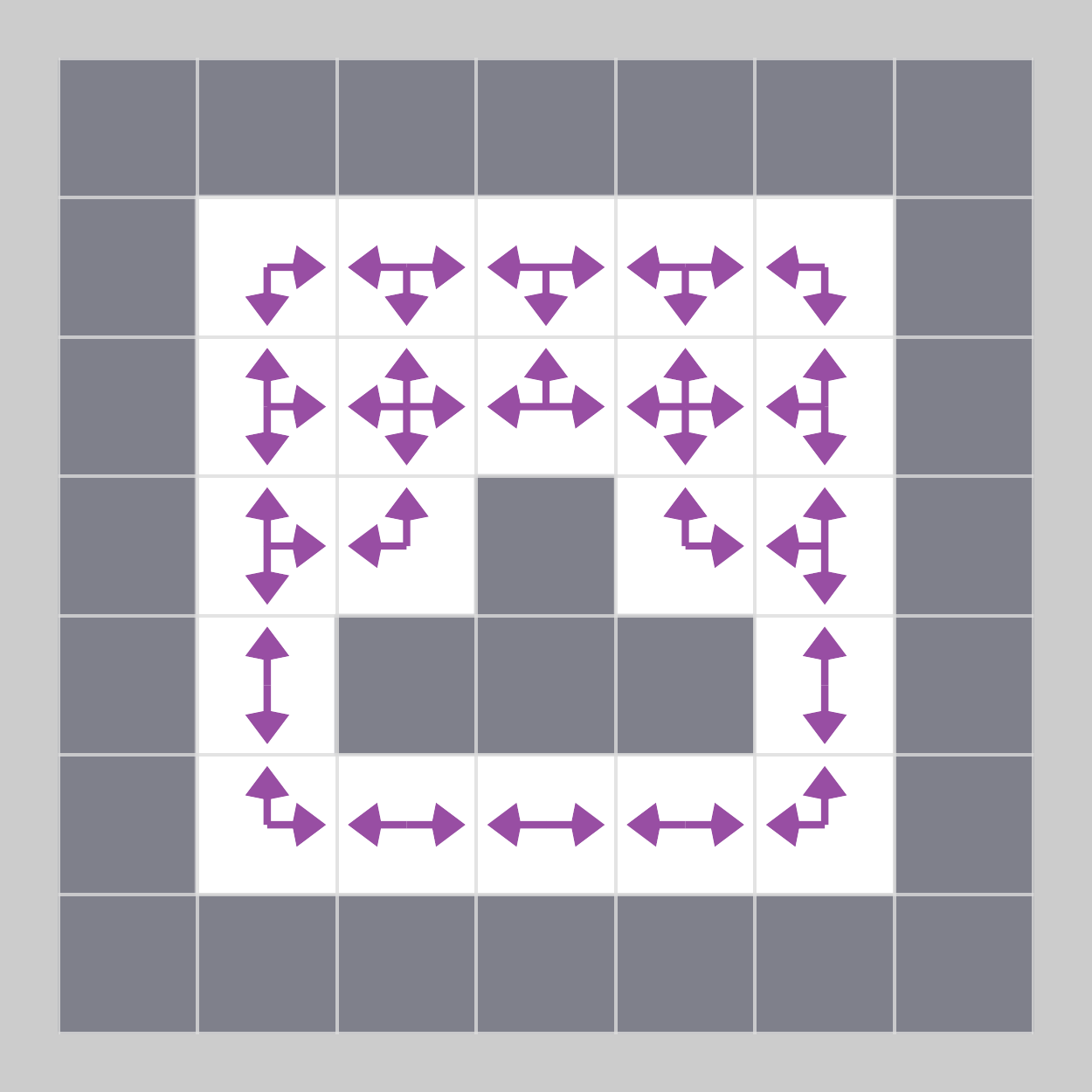}} 
    \hspace{0.4cm}
    \subfigure[One-room, Move Left Intent]{\label{fig:appendix_intent_left}\includegraphics[width=0.2\textwidth]{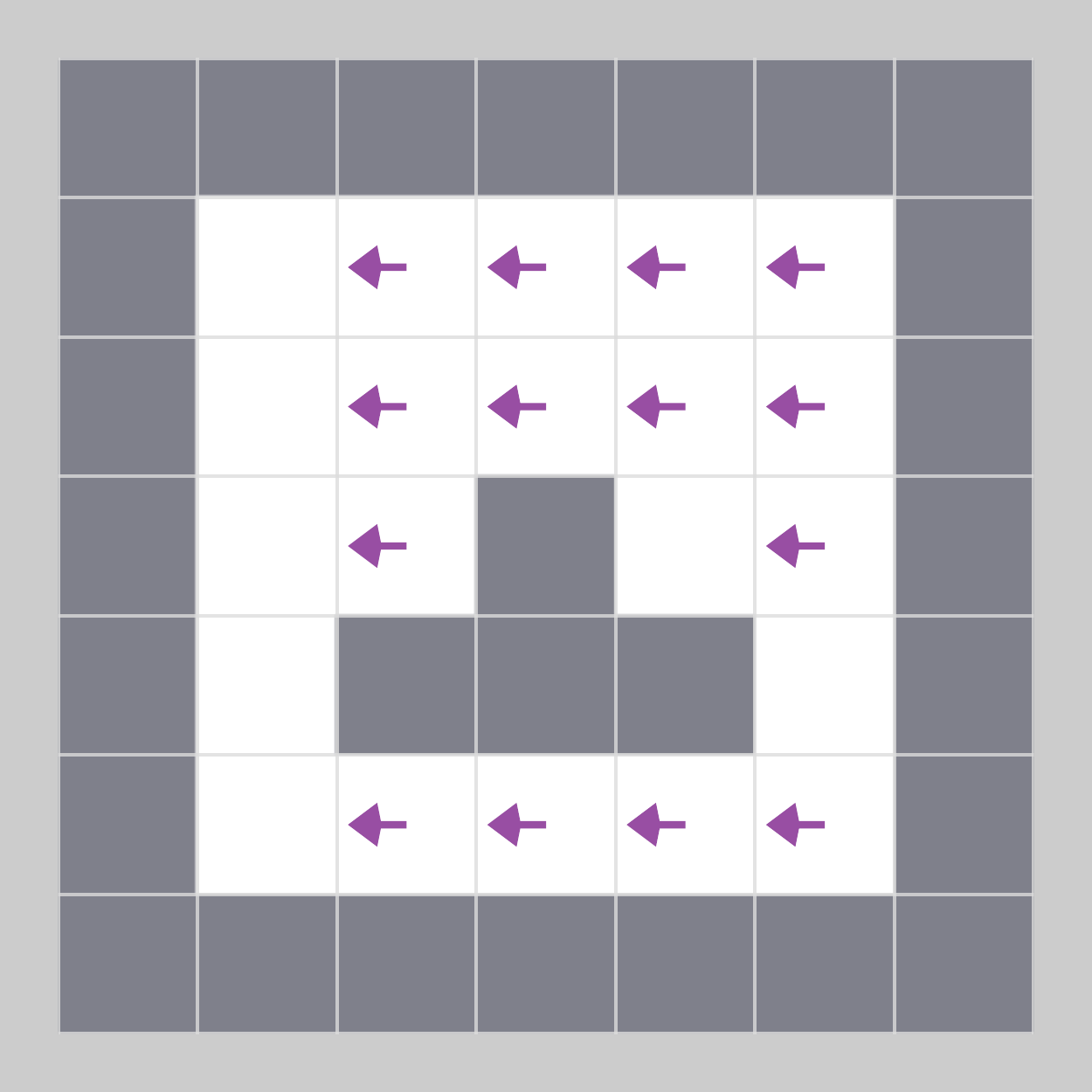}}
    \subfigure[One-room, Move Intent]{\label{fig:appendix_intent_changestate}\includegraphics[width=0.2\textwidth]{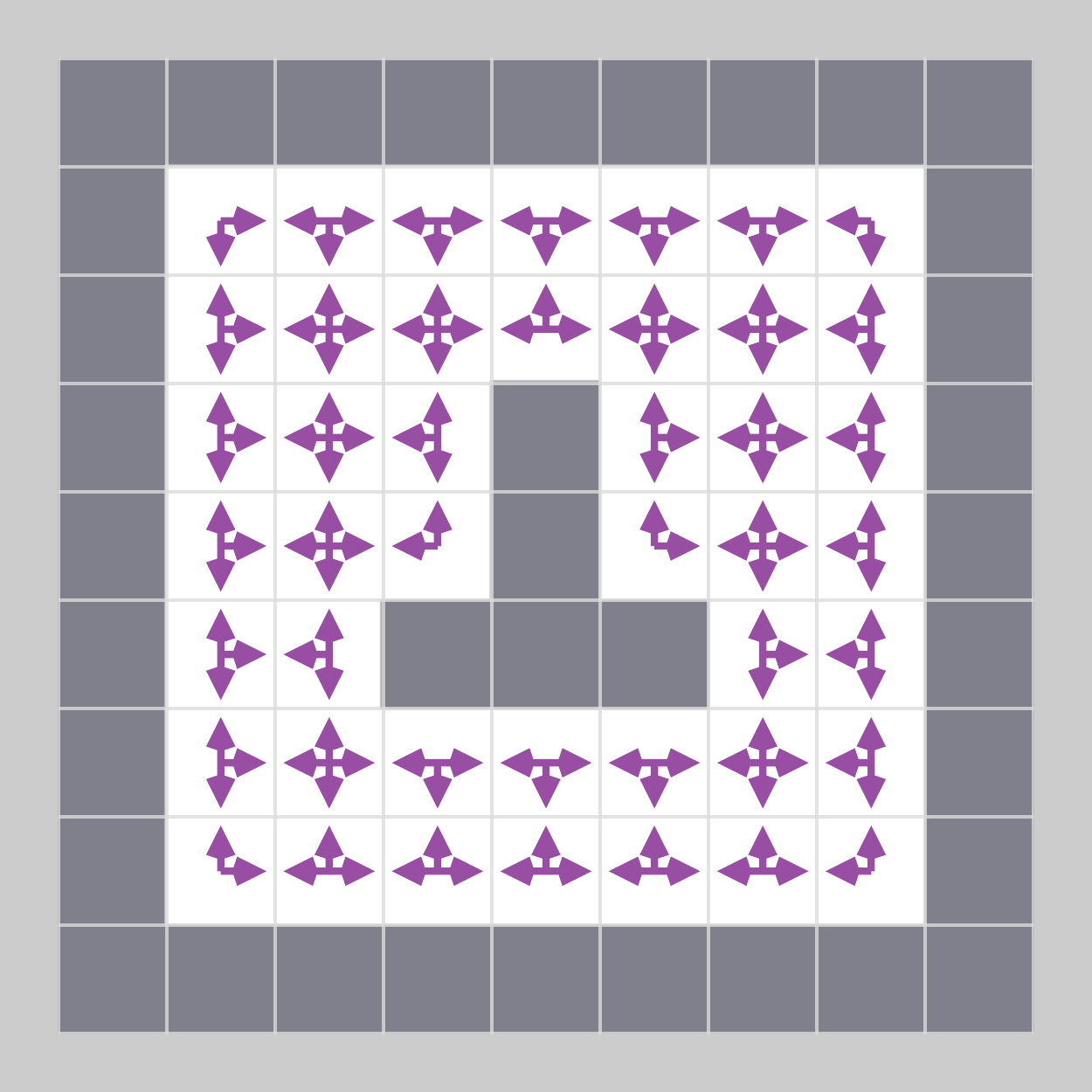}} 
    \hspace{0.4cm}
    \subfigure[One-room, Move Left Intent]{\label{fig:appendix_intent_left}\includegraphics[width=0.2\textwidth]{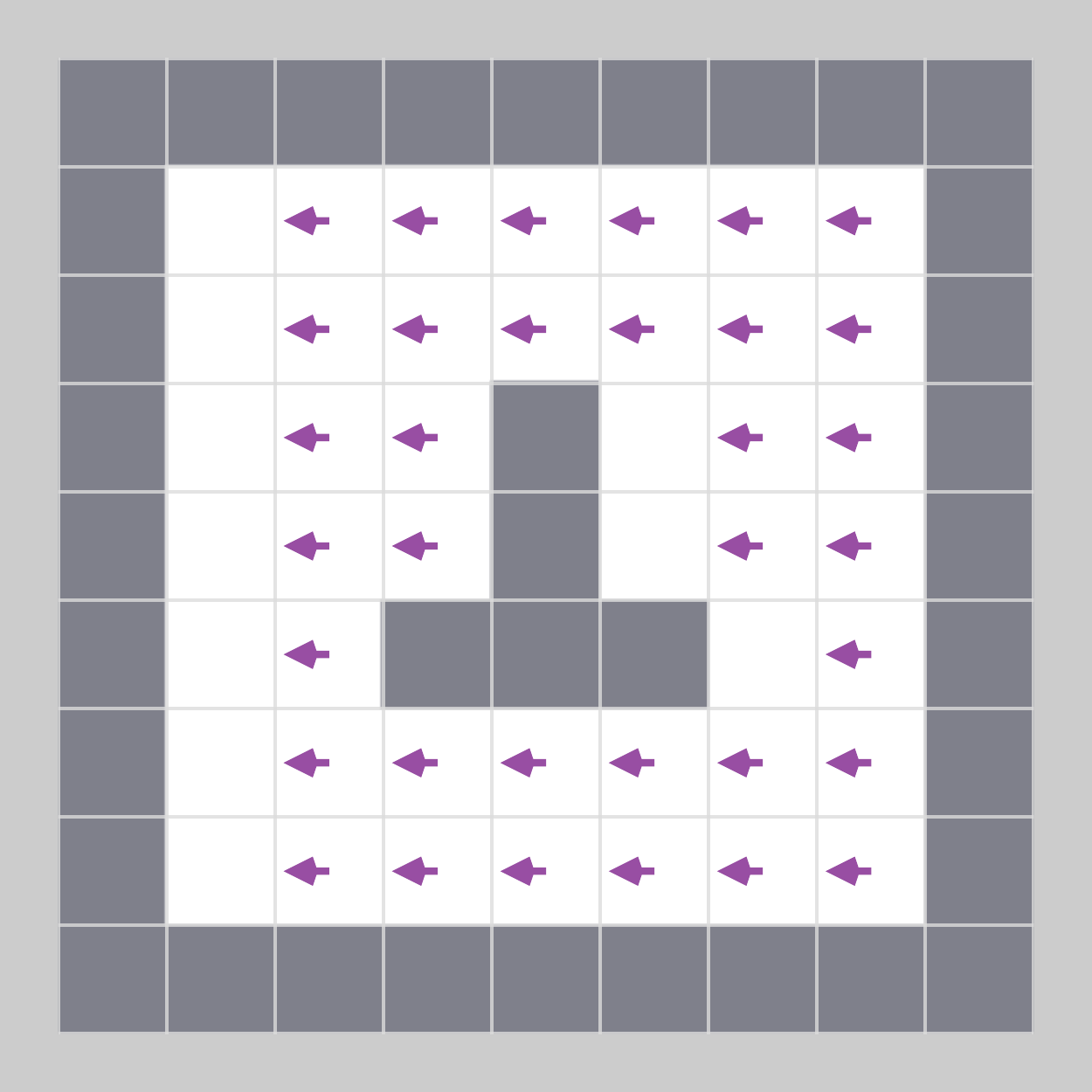}}
    \hspace{0.4cm}
    \subfigure[Pachinko, Move Left Intent]{\label{fig:appendix_pachinko_left}\includegraphics[width=0.2\textwidth]{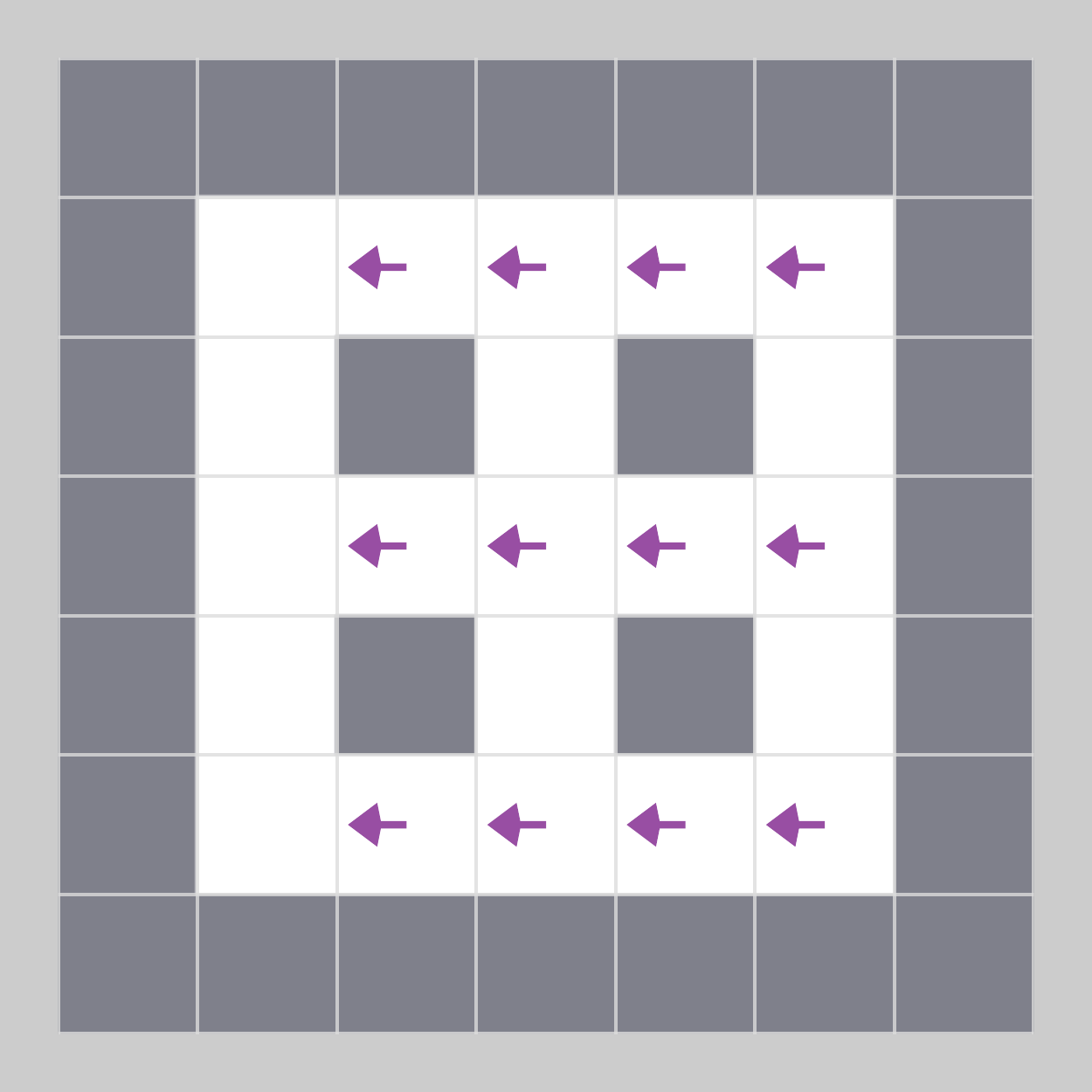}}
    \hspace{0.4cm}
    \subfigure[Pachinko, Move Intent]{\label{fig:appendix_pachinko_move}\includegraphics[width=0.2\textwidth]{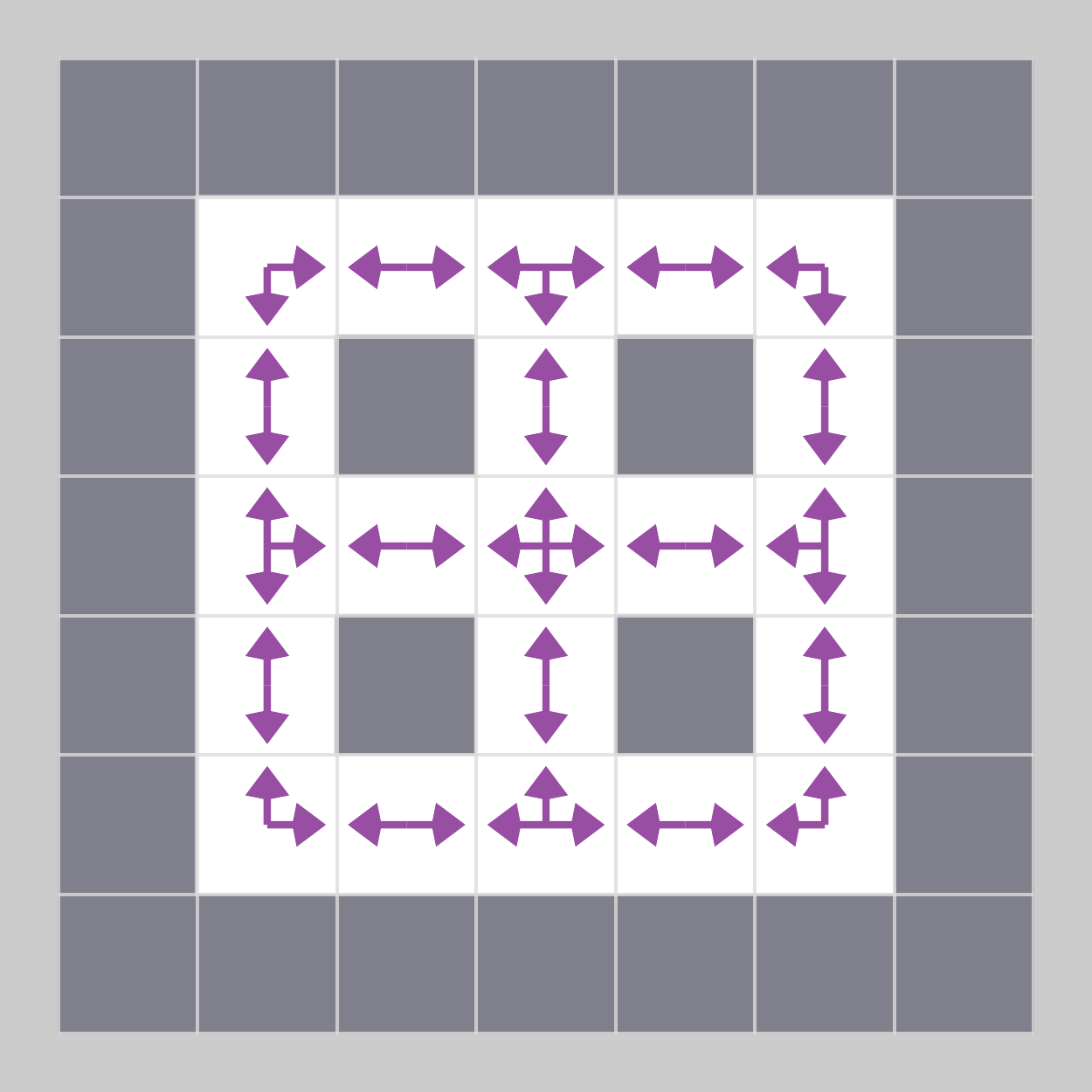}}
    \hspace{0.4cm}
    \subfigure[Pachinko, Move Up Intent]{\label{fig:appendix_pachinko_up}\includegraphics[width=0.2\textwidth]{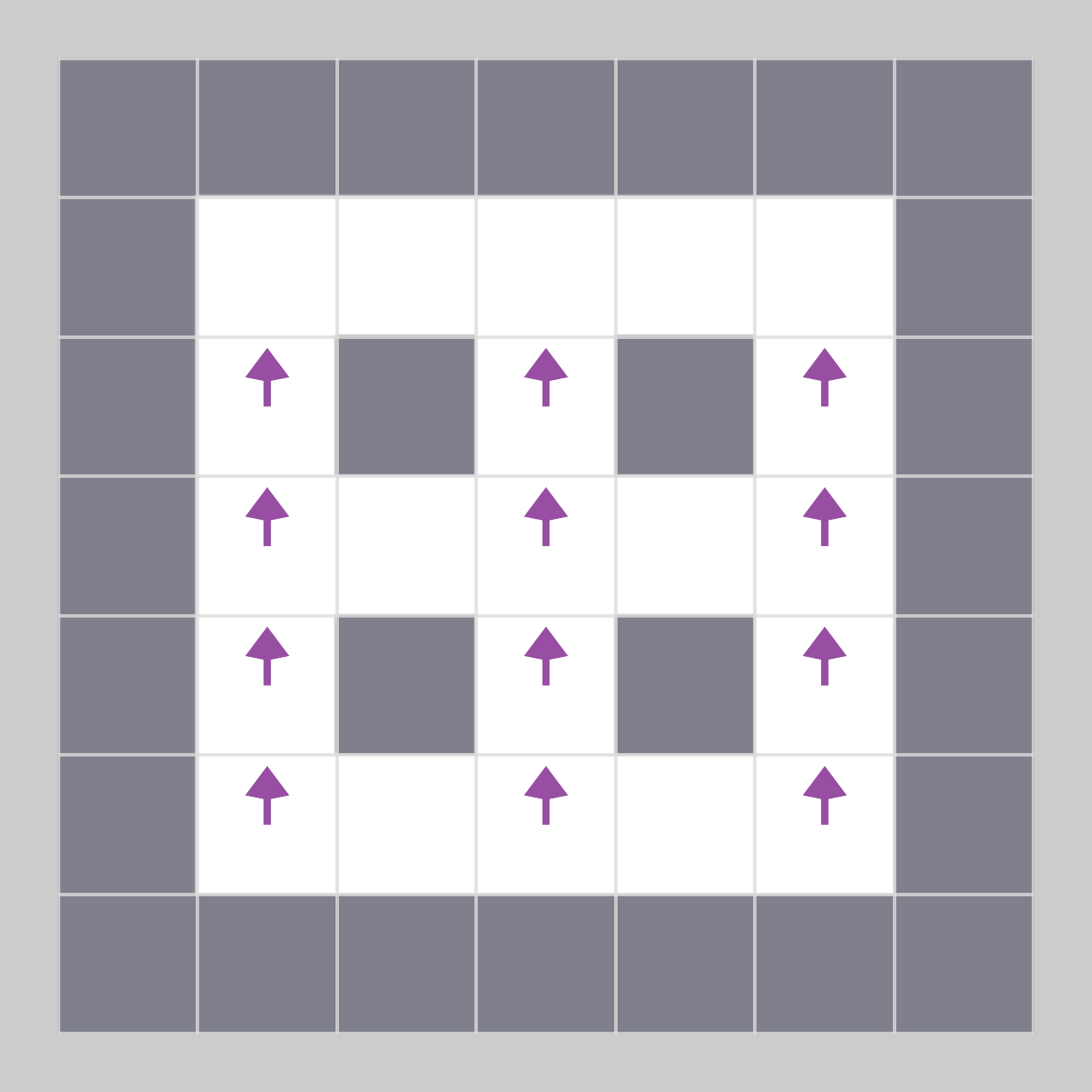}}
    \hspace{0.4cm}
    \subfigure[Pachinko, Move Intent]{\label{fig:appendix_pachinko_move}\includegraphics[width=0.2\textwidth]{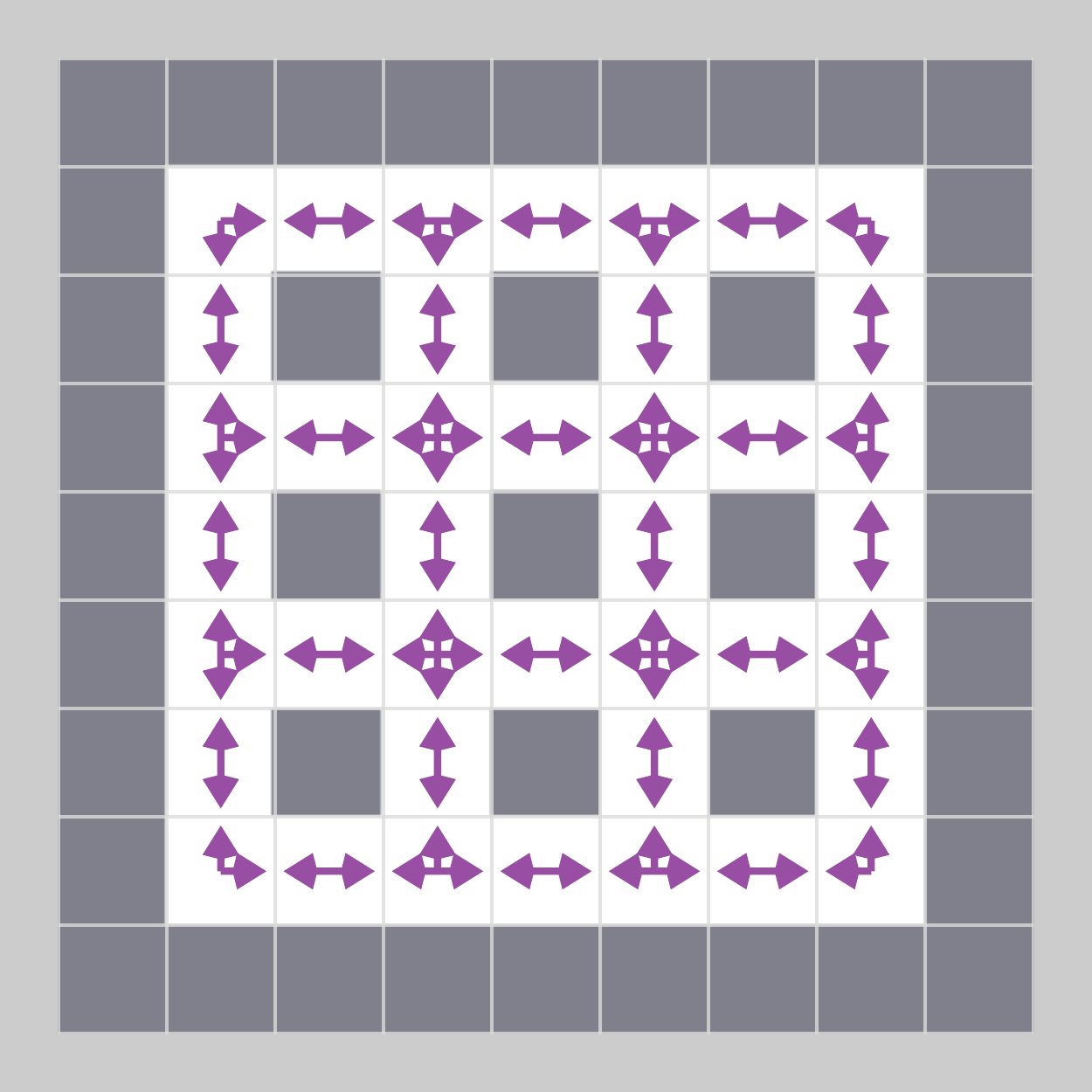}}
    \caption{\label{fig:appendix_Illustration}\textbf{Visualization of learned affordances in a variety of grid-worlds.} \textit{Affordances} and \textit{intents} are a general concept that can generalize across environments.}
    \end{center}
    \vskip -0.2in %
\end{figure}

\begin{figure}[h]
    \begin{center}
    \includegraphics[width=0.30\textwidth]{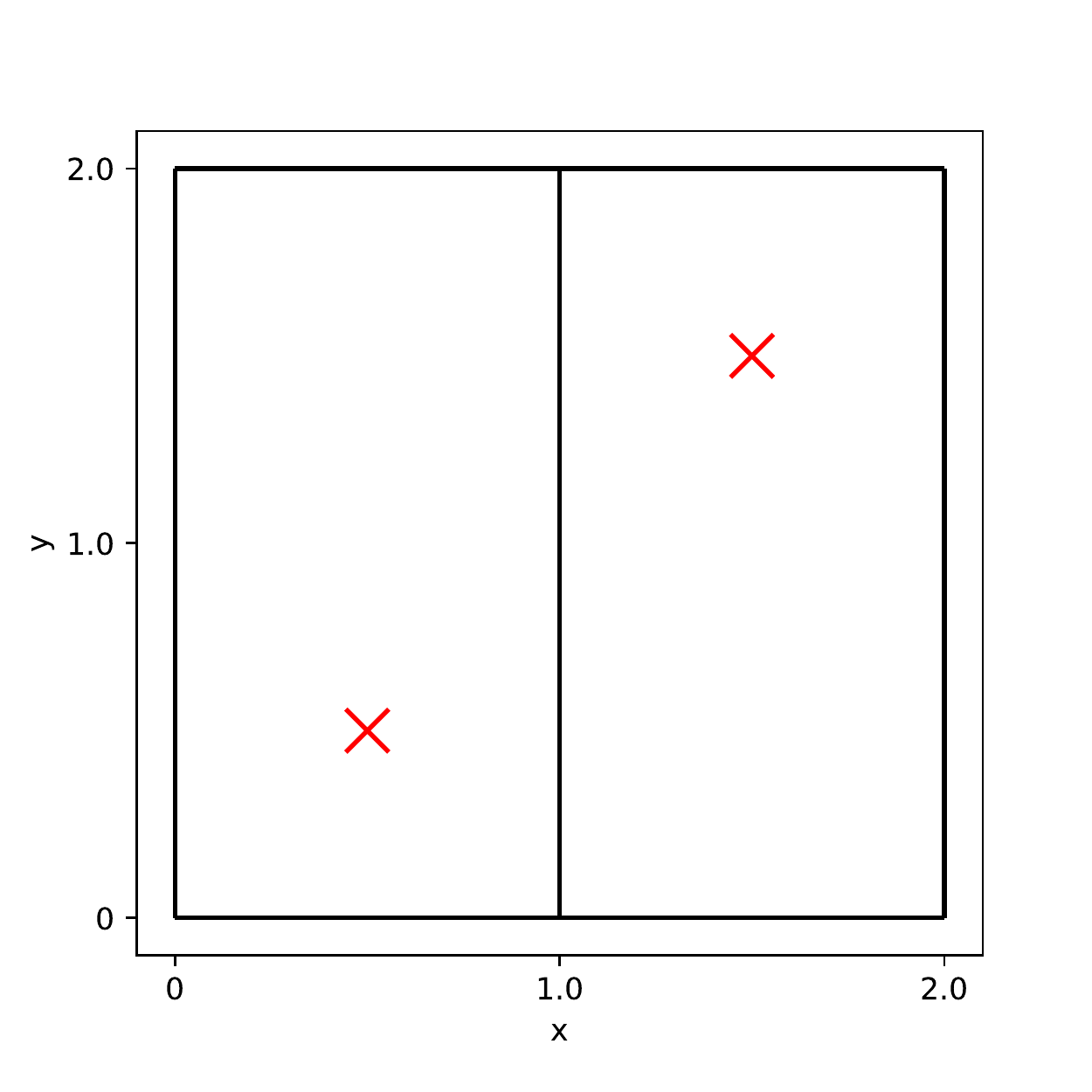}
    \caption{\label{fig:continuousgridworld}\textbf{Non-stationary Continuous World.} The continuous world has agent state represented by (x,y) coordinates. An impassable wall divides the world into two halves. The action space consists of displacements in the (x,y) directions. The two red crosses indicate the possible starting positions in the world. The agent will start around one cross for a fixed number of episodes before drifting toward the other. Upon reaching the next cross, the system reverses the starting state of the agent drifts towards the other.}
    \end{center}
\end{figure}

\begin{figure}[h]
    \begin{center}
    \subfigure[]{\includegraphics[width=0.5\textwidth]{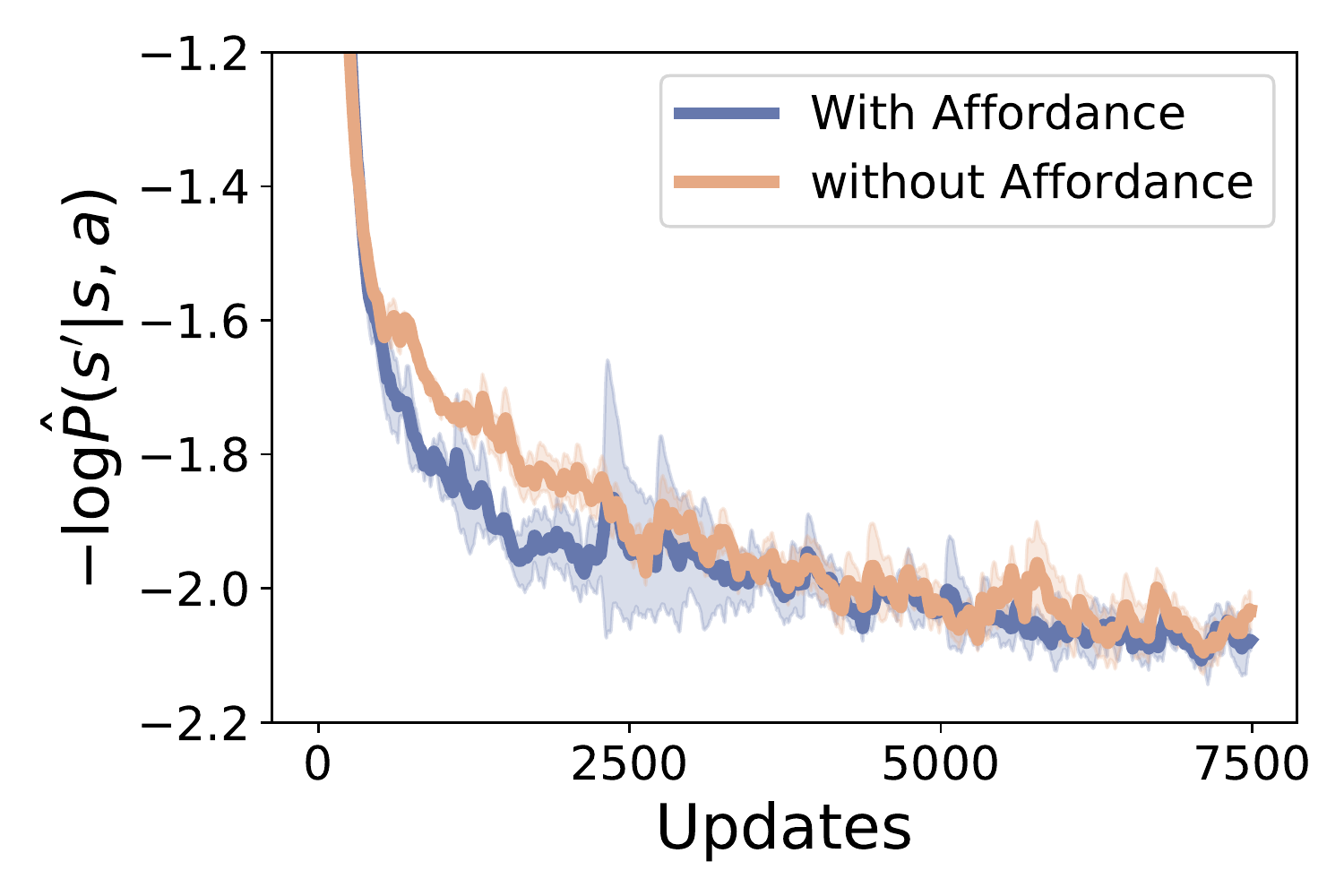}}
    \caption{\label{fig:average_learning_curves}\textbf{Learning curves for training a transition model with and without affordances.} Though the losses are similar, qualitative performance of the models is quite different. Shaded areas show standard error of the mean over 5 independent runs.}
    \end{center}
\end{figure}

\begin{figure}[h]
    \begin{center}
    \includegraphics[trim={0.0cm, 1.5cm, 0.6cm, 1.5cm},clip,width=0.9\textwidth]{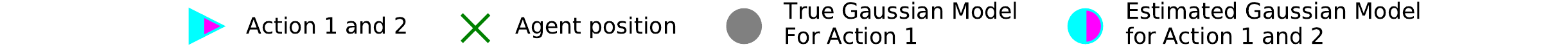}\\
    \includegraphics[width=0.25\textwidth]{{figures/continuous/replicated/no_labels/V3_model_empirical_demo_seed0_movement_noise0.1_P0.95_1.5_F0.2_0.0_affordanceFalse}.pdf}
    \includegraphics[width=0.25\textwidth]{{figures/continuous/replicated/no_labels/V3_model_empirical_demo_seed1_movement_noise0.1_P0.95_1.5_F0.2_0.0_affordanceFalse}.pdf}
    \includegraphics[width=0.25\textwidth]{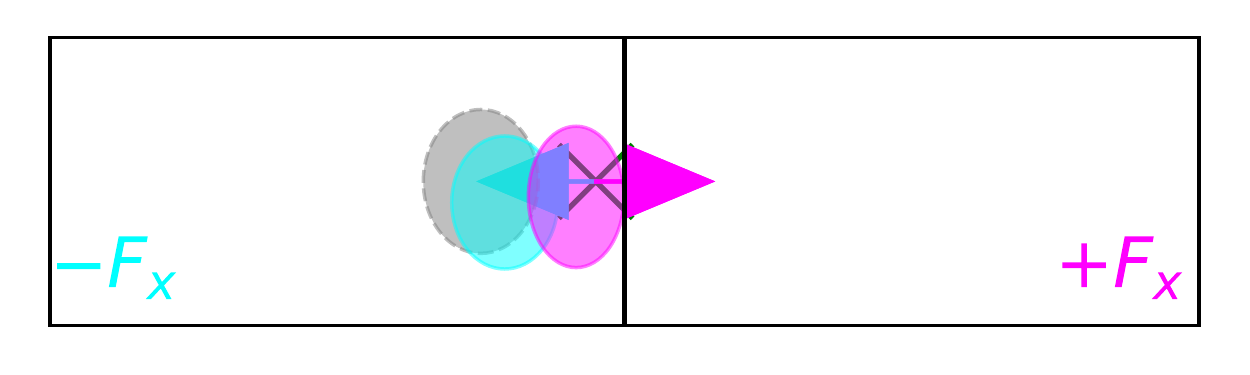}
    \includegraphics[width=0.25\textwidth]{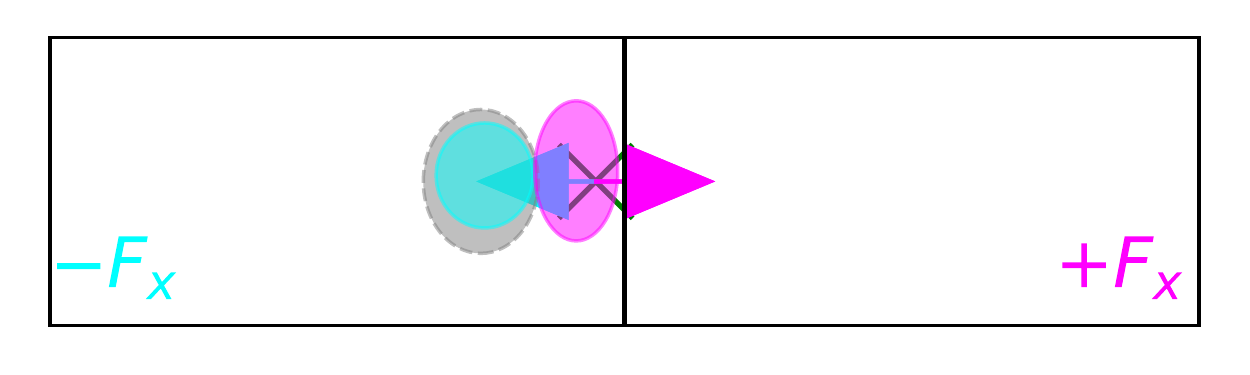}
    \includegraphics[width=0.25\textwidth]{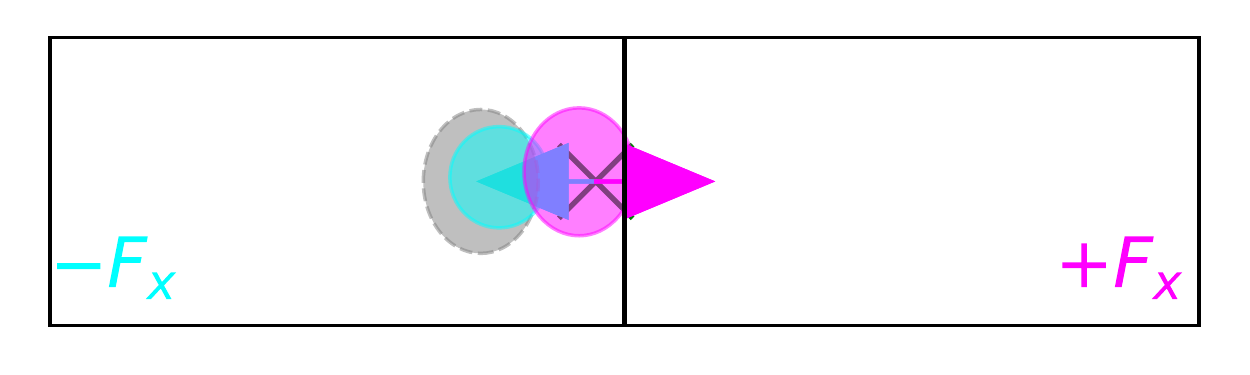}
    \caption{\label{fig:appendix_eval_vanilla_model}\textbf{Empirical visualization of a baseline model for five independent seeds.} 
    }
    \end{center}
\end{figure}

\begin{figure}[h]
    \begin{center}
    \includegraphics[trim={0.0cm, 1.5cm, 0.6cm, 1.5cm},clip,width=0.9\textwidth]{{figures/continuous/legends/legend_horizontal}.pdf}\\
    \includegraphics[width=0.25\textwidth]{{figures/continuous/replicated/no_labels/V3_model_empirical_demo_seed0_movement_noise0.1_P0.95_1.5_F0.2_0.0_affordanceTrue}.pdf}
    \includegraphics[width=0.25\textwidth]{{figures/continuous/replicated/no_labels/V3_model_empirical_demo_seed1_movement_noise0.1_P0.95_1.5_F0.2_0.0_affordanceTrue}.pdf}
    \includegraphics[width=0.25\textwidth]{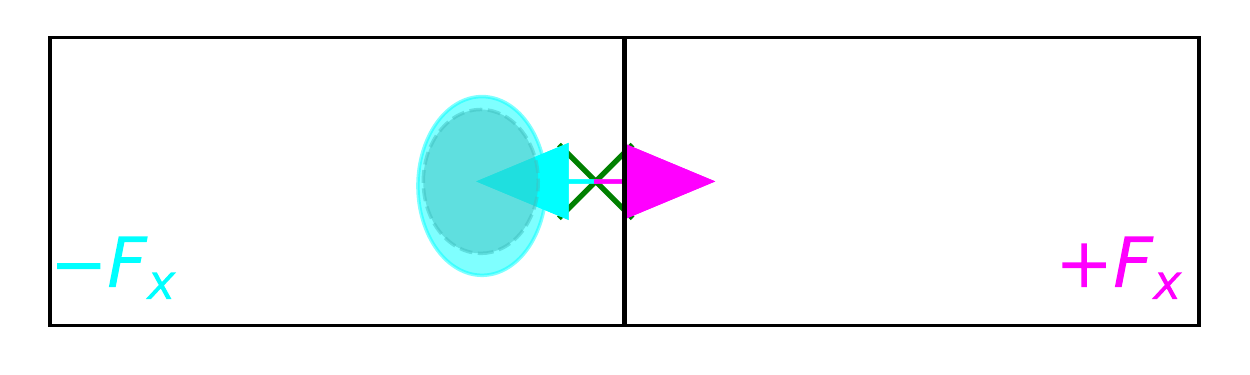}
    \includegraphics[width=0.25\textwidth]{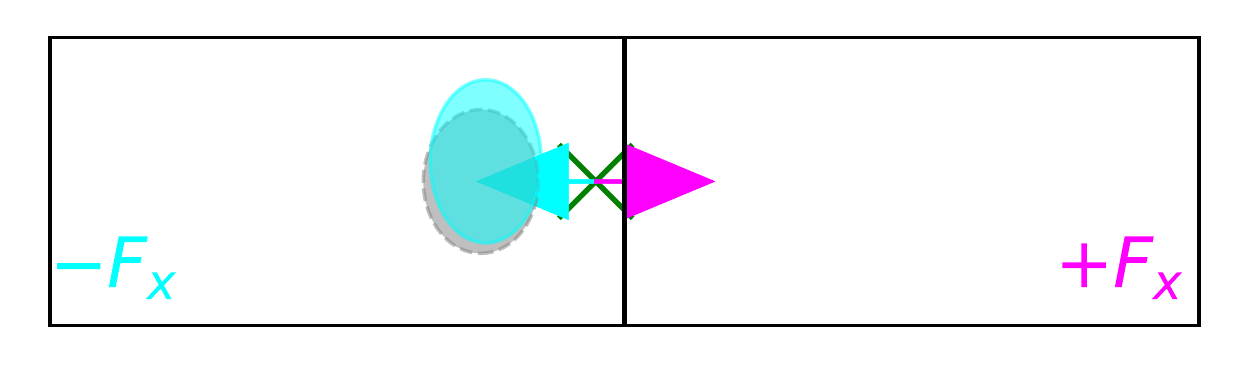}
    \includegraphics[width=0.25\textwidth]{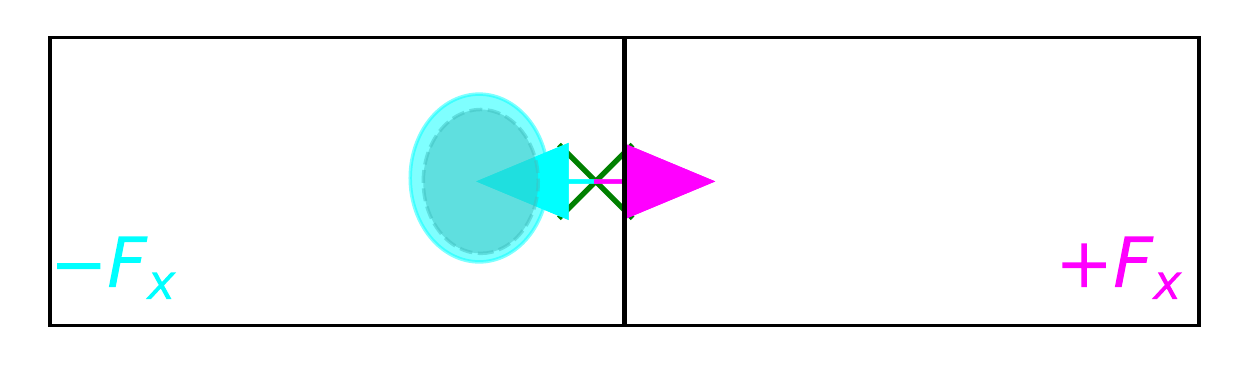}
    \caption{\label{fig:appendix_eval_affordance_model}\textbf{Empirical visualization of an affordance-aware partial model for five independent seeds.} 
    }
    \end{center}
\end{figure}

\begin{figure}[h]
    \begin{center}
    \includegraphics[trim={0.0cm, 1.5cm, 0.6cm, 1.5cm},clip,width=0.9\textwidth]{{figures/continuous/legends/legend_horizontal}.pdf}\\
    \includegraphics[width=0.25\textwidth]{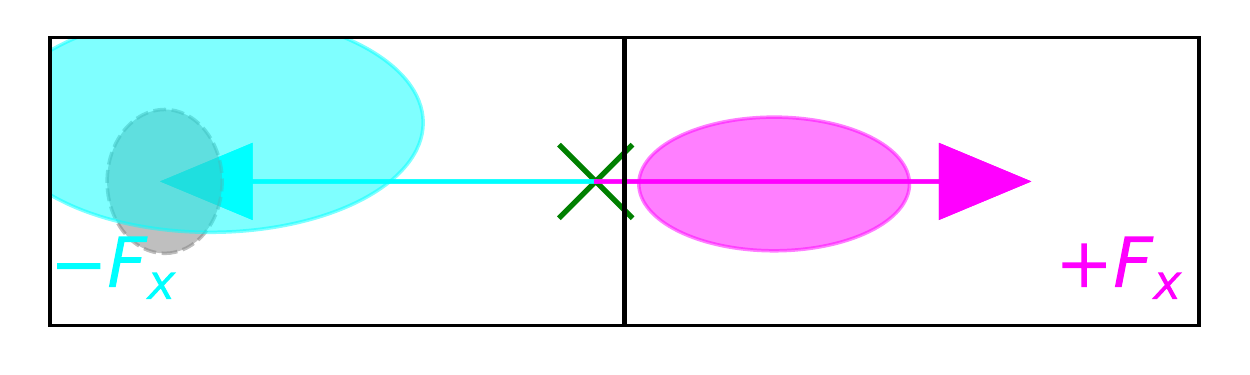}
    \includegraphics[width=0.25\textwidth]{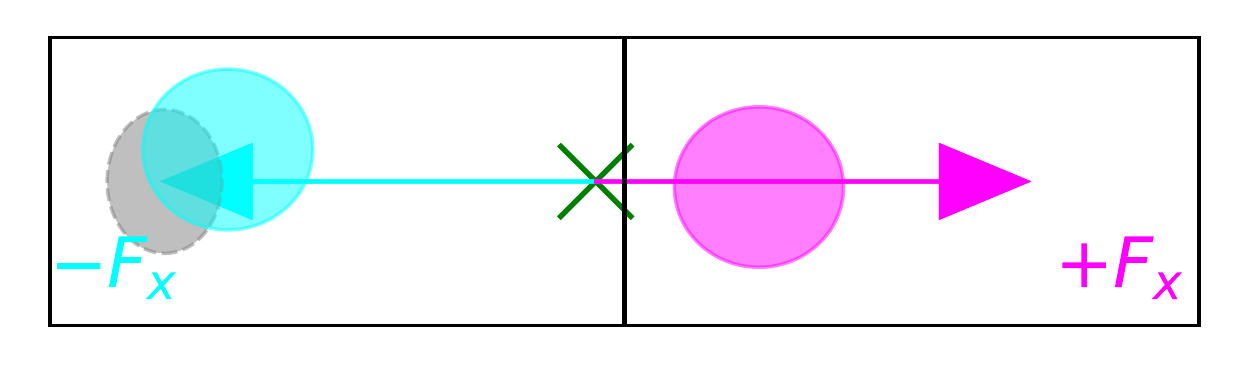}
    \includegraphics[width=0.25\textwidth]{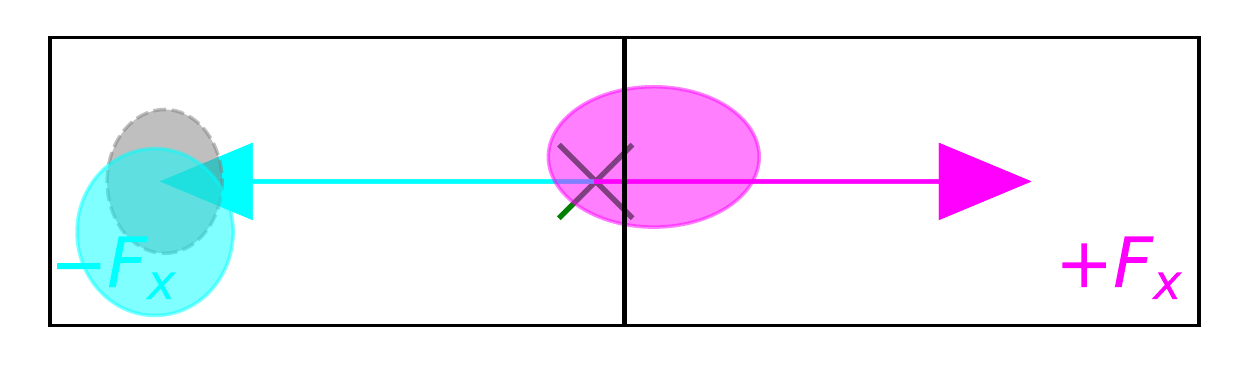}
    \includegraphics[width=0.25\textwidth]{{figures/continuous/replicated/no_labels/V3_model_empirical_demo_seed2_movement_noise0.1_P0.95_1.5_F0.75_0.0_affordanceFalse}.pdf}
    \includegraphics[width=0.25\textwidth]{{figures/continuous/replicated/no_labels/V3_model_empirical_demo_seed3_movement_noise0.1_P0.95_1.5_F0.75_0.0_affordanceFalse}.pdf}
    \caption{\label{fig:appendix_generalization_vanilla_model}\textbf{Empirical visualization of generalization for a baseline model for four independent seeds.}
    }
    \end{center}
\end{figure}

\begin{figure}[h]
    \begin{center}
    \includegraphics[trim={0.0cm, 1.5cm, 0.6cm, 1.5cm},clip,width=0.9\textwidth]{{figures/continuous/legends/legend_horizontal}.pdf}\\
    \includegraphics[width=0.25\textwidth]{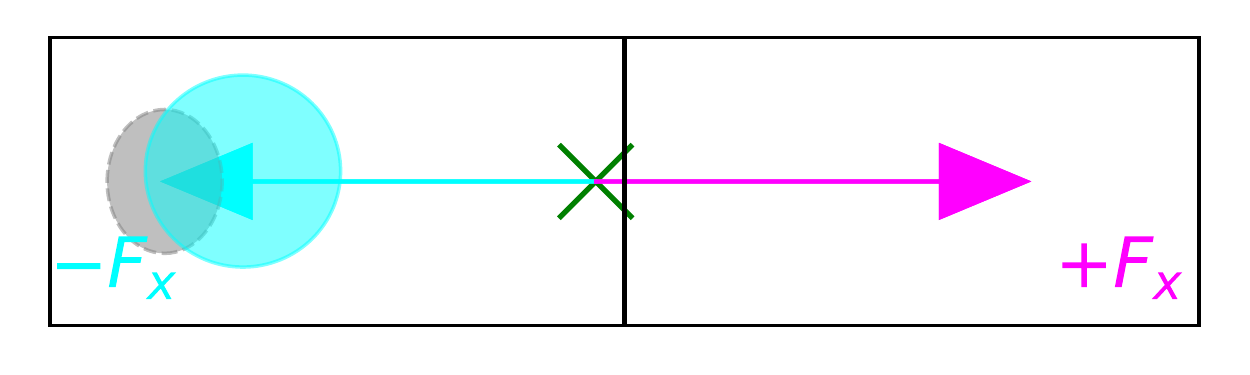}
    \includegraphics[width=0.25\textwidth]{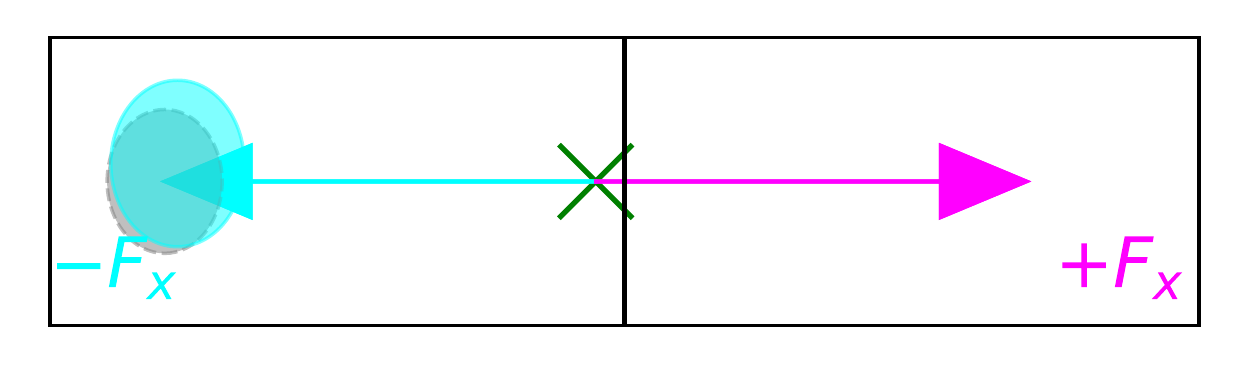}
    \includegraphics[width=0.25\textwidth]{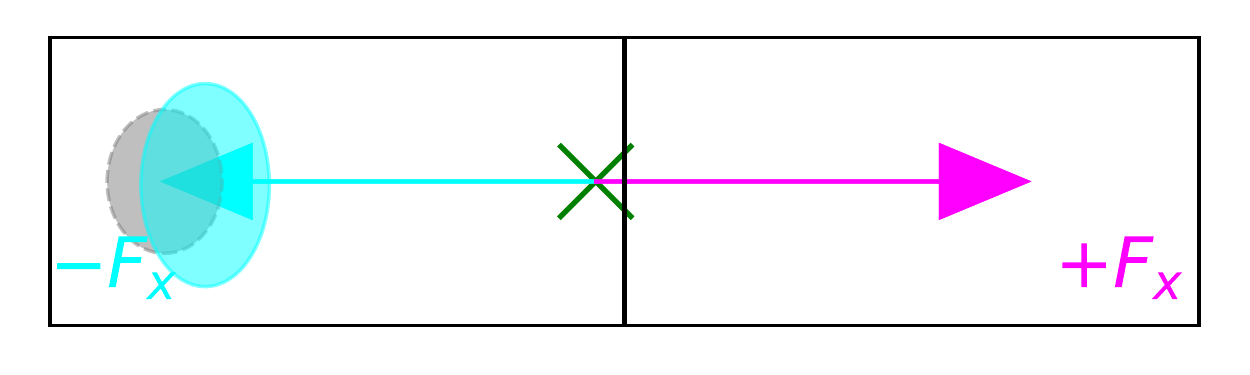}
    \includegraphics[width=0.25\textwidth]{{figures/continuous/replicated/no_labels/V3_model_empirical_demo_seed2_movement_noise0.1_P0.95_1.5_F0.75_0.0_affordanceTrue}.pdf}
    \includegraphics[width=0.25\textwidth]{{figures/continuous/replicated/no_labels/V3_model_empirical_demo_seed3_movement_noise0.1_P0.95_1.5_F0.75_0.0_affordanceTrue}.pdf}
    \caption{\label{appendix_fig:generalization_affordance_model}\textbf{Empirical visualization of generalization for an affordance-aware partial model for four independent seeds.}
    }
    \end{center}
\end{figure}

\subsection{Learning Affordances: Additional Details}
\subsubsection{Learning affordances in discrete environments}
Following the learning approach described in the Sec~\ref{sec:learning_aff} of the main paper, we are also able to learn affordances in a variety of discrete grid-world like environments. We demonstrate the learned affordances for different intent specifications in Fig~\ref{fig:appendix_Illustration}. Note that for a given intent, affordances learned are invariant to various factors such as size of the grid-world, the location of the walls, etc. and capture the underlying dynamics consistently across environments.

\begin{algorithm}[tb]
   \caption{Pseudo code: Affordance-aware model learning}
   \label{alg:aff_model_learning}
\begin{algorithmic}
   \STATE {\bfseries Require:} Collection of intents $\mathcal{I}$.
   \STATE {\bfseries Require:} Environment, $E$.
   \STATE {\bfseries Require:} Intent completion function $c$.
   \STATE {\bfseries Require:} Affordance-classifier $A_\theta$.
   \STATE {\bfseries Require:} Generative Model $P_\phi$.
   \STATE {\bfseries Require:} Data collection policy, $\pi$.
   \STATE {\bfseries Input:} Number of transitions $n$, thresholds $\delta,k$, number of training steps $N$.
\FOR{$i=0\ldots N$ } 

   \STATE \textbf{1.Calculate affordance-classifier loss:}
   \STATE $\{(s, a, s')\}_n \leftarrow$ collect\_trajectories($\pi$, E, n)
   \STATE $\{c\}_n \leftarrow$ get\_intent\_completions($\{(s, a, s', i)\}_n$)
   \STATE $\mathcal{D}\leftarrow \{(s, a, s')\}_n\cup\{c\}_n$
   
  \textbf{2. Calculate affordance-classifier loss:}
   \STATE $\mathcal{O}_A(\theta) \leftarrow \sum_{(s,a,s')\in\mathcal{D}}\sum_{I\in \mathcal{I}}-c(s,a,s',I)\log A_\theta(s,a,I)$

  \textbf{3. Calculate gradient and apply update to $\theta$.}
   \STATE $\theta \leftarrow \theta + \nabla_\theta\mathcal{O}_A(\theta)$
   
  \textbf{4. Calculate model loss:}
   \STATE $\mathcal{O}_{\text{aff}}(\phi) \leftarrow \sum_{(s,a,s')\in \mathcal{D}}\mathbbm{1}\Big[\max_{\forall I\in\mathcal{I}}A_\theta(s,a,I) > k\Big]\log P_\phi(s'|s,a)$
   
  \textbf{5. Calculate gradient and apply update to $\phi$.}
   \STATE $\phi \leftarrow \phi + \nabla_\phi\mathcal{O}_{\text{aff}}(\phi)$
\ENDFOR

\end{algorithmic}
\end{algorithm}

\subsubsection{Learning affordance-aware models under model class restrictions}
\label{sec:appendix:restricted_models}
In Sec~\ref{sec:model_learning} we investigated the qualitative behavior of an affordance-aware model. To better understand the convergence and generalization behavior, we conduct an illustrative experiment with a restricted model class in the Continuous World (Fig~\ref{fig:continuousgridworld}). In particular, we consider a restricted class of transition models that aim to capture displacement only, relative to the current position, and do not have access to state information: $P_\phi(s'|a,s)=\mathcal{N}(s + \mu_\phi(a), \sigma_\phi(a))$. Where $\mu_\phi$ and $\sigma_\phi$ are learned outputs of a neural networks. Learning this model might be difficult in general due to the walls: The model has no information about the current state $s$ and therefore, will not be able to predict that near the walls $\mu_\phi(a)$ should equal 0. In contrast, an affordance classifier implicitly encodes the wall information in $A_\theta$ and the masking should allow the model to learn the rule that $\mu_\phi(a) \approx a$ and $\sigma_\phi(a) \approx 0.1$. Here we use a linear approximator for $\mu_\phi$ and $\sigma_\phi$.

After training %
for 5000 updates, the affordance-aware model reaches a lower loss than the baseline model (Fig~\ref{fig:restricted_class_average_learning_curves}). To quantitatively understand the generalization of the models, we evaluate the mean prediction for a range of actions, $F_x$, over 5 independent seeds. The baseline model learns that $\sigma\approx 0.148$ while the affordance-aware model estimates $\sigma\approx 0.133$. Both baseline and affordance-aware models systematically under-estimate the mean (Fig.\ref{fig:restricted_ood_action}). However, the baseline model is more cautious, and predicts means that are even smaller than the affordance-aware model, since it needs to account for not being able to move near the walls.
The predicted mean from the affordance-aware model during the last 200 updates is significantly closer to the true value, with an average error of $0.038$ compared to the baseline model's  error of $0.065$ (Student's T-test, $p\approx 10^{-28}$).

\begin{figure}[h]
    \begin{center}
    \subfigure[]{\includegraphics[width=0.5\textwidth]{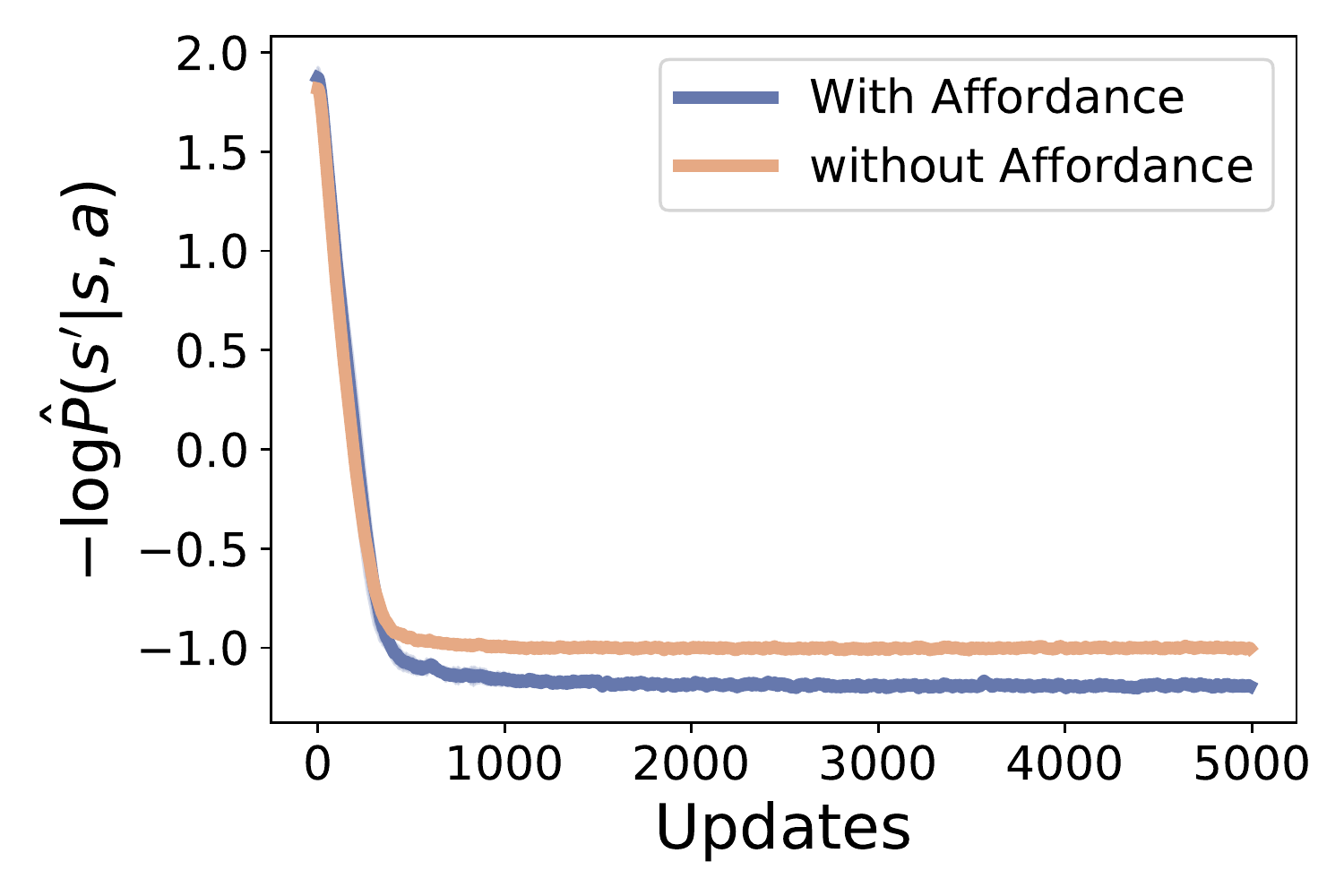}}
    \caption{\label{fig:restricted_class_average_learning_curves}\textbf{Learning curves for training a restricted transition model with and without affordances.} Shaded areas show standard error of the mean over 5 independent runs.}
    \end{center}
\end{figure}

\begin{figure}[h]
    \begin{center}
    \subfigure[$F_x=0.1$]{\includegraphics[width=0.4\textwidth]{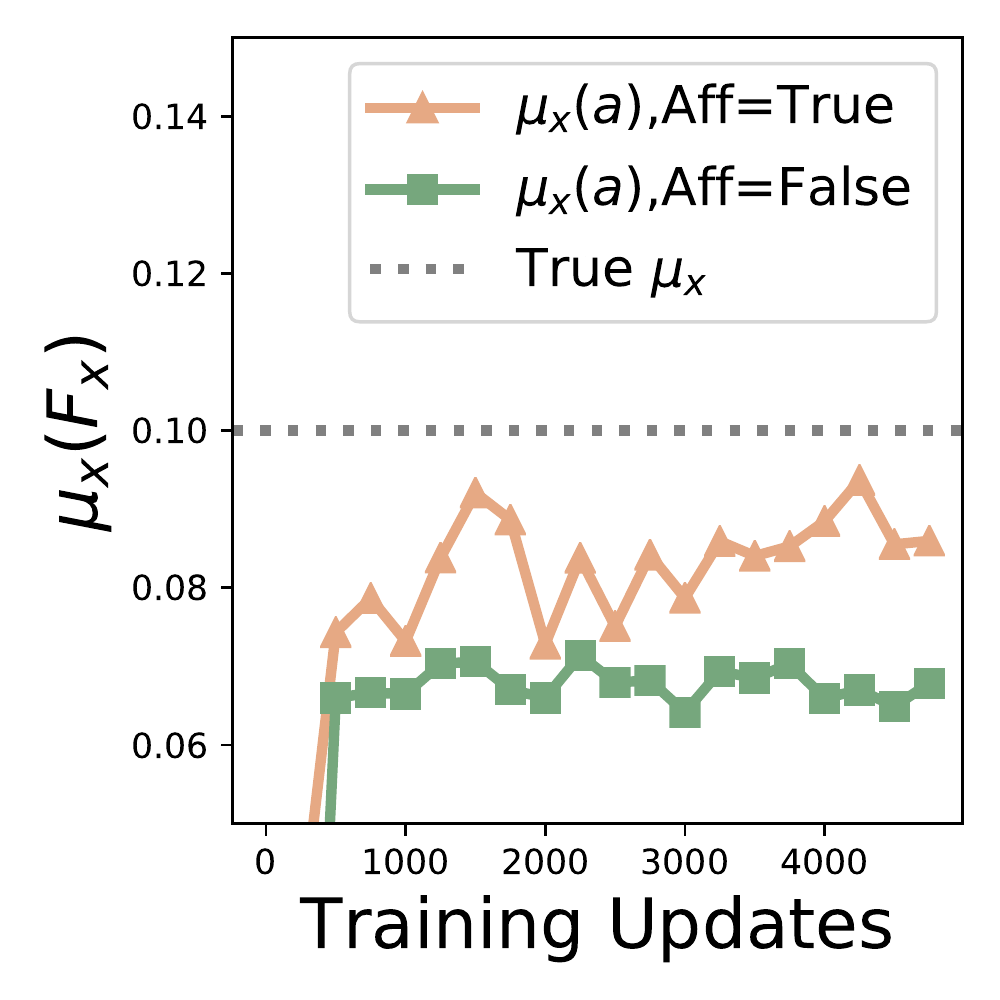}}
    \subfigure[$F_x=0.5$]{\includegraphics[width=0.4\textwidth]{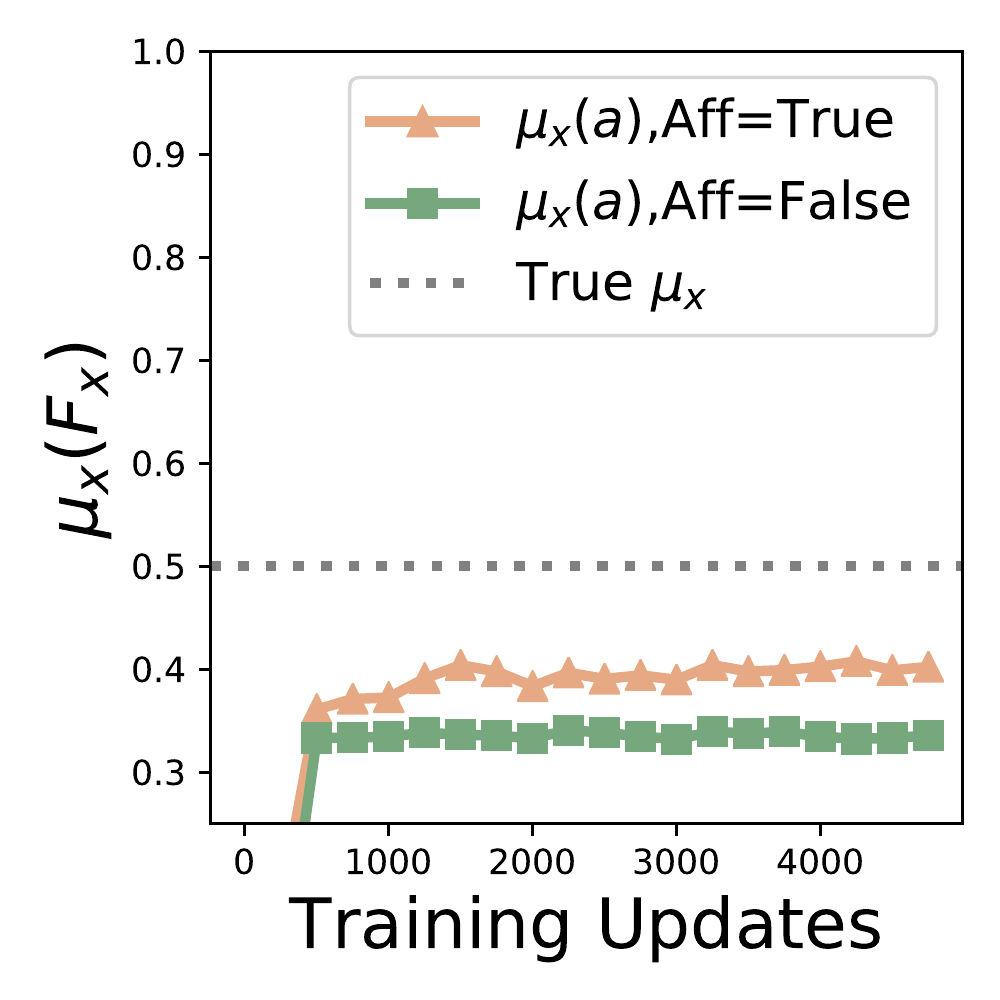}}
    \subfigure[$F_x=0.75$]{\includegraphics[width=0.4\textwidth]{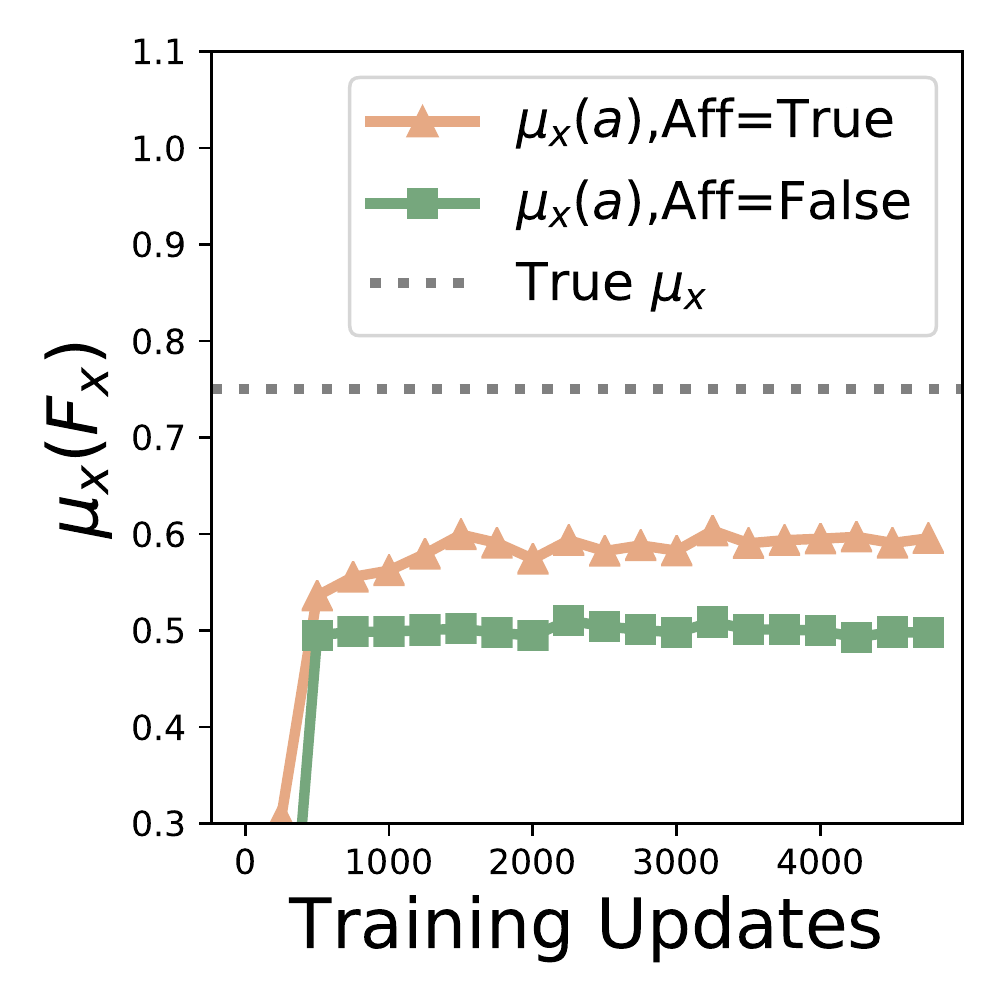}}
    \subfigure[$F_x=1.0$]{\includegraphics[width=0.4\textwidth]{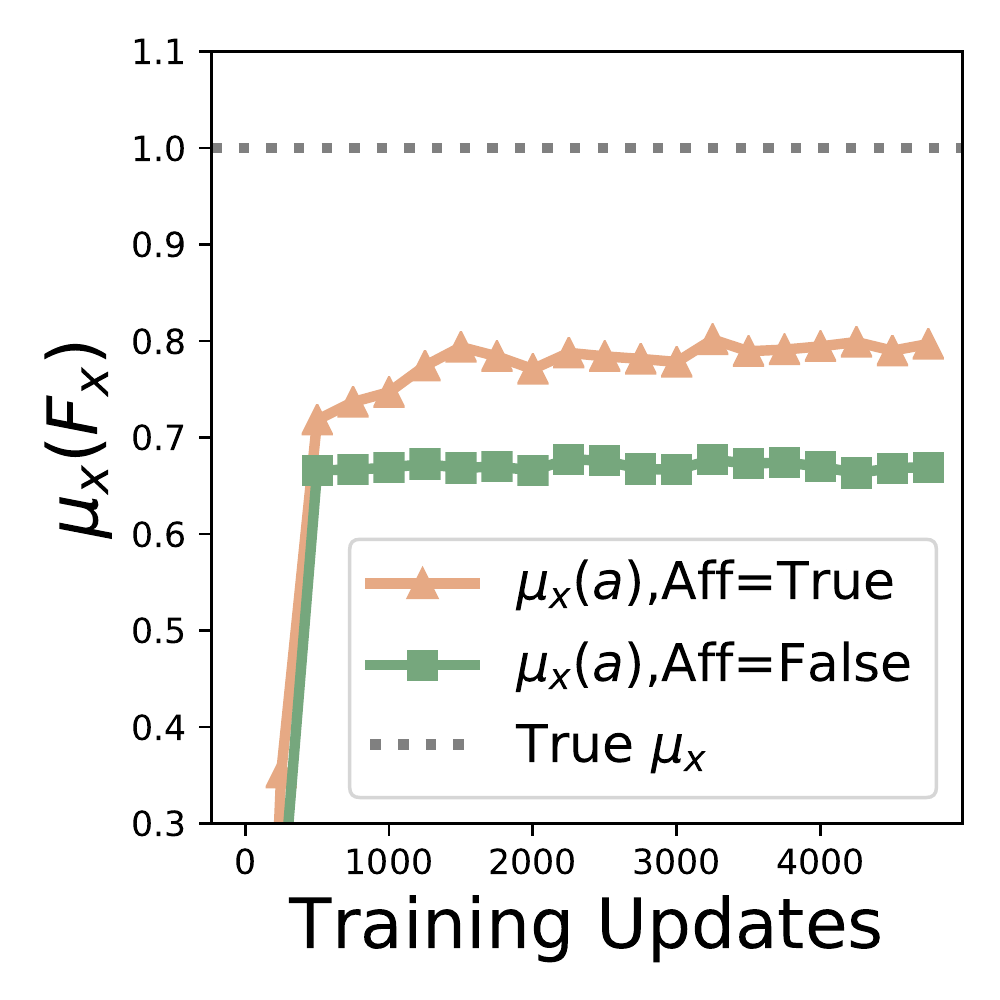}}
    \caption{\label{fig:restricted_ood_action}\textbf{Model prediction accuracy.} Independent runs of the models described in Sec.~\ref{sec:appendix:restricted_models}. Dotted lines show the true prediction. The curves show the mean prediction of the models during training. We here show the predictions by both model with affordances (Aff=True), model without affordances (Aff=False), and ground truth (True $\mu_x$) with a dotted line.}
    \end{center}
\end{figure}

\end{document}